\documentclass[twoside,11pt]{article}

\usepackage[preprint]{jmlr2e}
\usepackage{commath}
\usepackage{algorithm}
\usepackage{algpseudocode}
\usepackage{booktabs}
\usepackage{multirow}
\usepackage{xcolor}
\usepackage{lastpage}

\usepackage{amsfonts,amsbsy,amsmath,bbm}
\usepackage[nameinlink,capitalize,noabbrev]{cleveref}

\crefname{assumption}{Assumption}{Assumptions}

\newcommand{\KL}{\mathrm{KL}}
\newcommand{\kl}{\mathrm{KL}}

\def\ind{\mathbbm{1}}

\def\d{\textup{d}}
\def\e{\textup{e}}
\def\E{\mathbb{E}}
\def\R{\mathbb{R}}

\def\p{\mathsf{p}}

\DeclareMathOperator*{\argmin}{arg\,min}

\def\Unif{\texttt{\textup{Unif}}}

\def\N{\texttt{\textup{N}}}

\newcommand{\cA}{\mathcal{A}}
\newcommand{\cB}{\mathcal{B}}
\newcommand{\cC}{\mathcal{C}}
\newcommand{\cD}{\mathcal{D}}

\newcommand{\cI}{\mathcal{I}}

\newcommand{\cL}{\mathcal{L}}

\newcommand{\cQ}{\mathcal{Q}}
\newcommand{\cR}{\mathcal{R}}

\newcommand{\cX}{\mathcal{X}}

\jmlrheading{0}{2026}{1--\pageref{LastPage}}{}{}{26-0000}{Jungbin Jun and Ilsang Ohn}
\ShortHeadings{Bayesian Online Expert Aggregation}{Jun and Ohn}
\firstpageno{1}

\begin{document}

\title{Adaptive Bayesian Online Learning via Expert Aggregation}

\author{\name Jungbin Jun \email jungbini03@inha.edu \\
       \addr Department of Statistics\\
       Inha University\\
       Incheon, 22212, South Korea
       \AND
       \name Ilsang Ohn\textsuperscript{*} \email ilsang.ohn@inha.ac.kr \\
       \addr Department of Statistics\\
       Inha University\\
       Incheon, 22212, South Korea}

\maketitle

\begingroup
\renewcommand{\thefootnote}{\fnsymbol{footnote}}
\footnotetext[1]{Corresponding author}
\endgroup

\begin{abstract}
Bayesian online learning promises uncertainty-aware prediction on data streams, but its performance hinges on inferential choices, including learning rates, prior distributions and variational families, which are usually fixed before seeing the stream. We address this by treating Bayesian update rules as experts and aggregating the Bayesian experts according to sequential predictive losses. We prove that the resulting aggregate competes with the best expert in hindsight at an aggregation cost determined by how each expert's per-round performance is evaluated. We instantiate the framework in online conformal inference and Gaussian process regression. The conformal inference application yields a smoothed Bayesian counterpart of adaptive conformal inference with long-run randomized coverage, while the Gaussian process application gives an oracle inequality in cumulative predictive Kullback-Leibler risk and adaptation to unknown H\"older smoothness up to logarithmic factors. Experiments show that the aggregate tracks strong experts without oracle expert selection.
\end{abstract}


\section{Introduction}
Modern predictive systems often require inference to proceed before the full dataset is available. Bayesian inference is well suited to this regime because it updates recursively: each posterior summarizes the information accumulated so far and becomes the prior for the next update. In generalized Bayes, this recursion extends beyond likelihood-based updates: posterior mass is sequentially reweighted according to a chosen loss, providing a natural bridge between Bayesian updating and online learning \citep{cherief2019generalization,alquier2021non, haddouche2022online, winter2026sequential}.

A persistent difficulty, however, is that Bayesian online learning still requires committing to inferential choices that can strongly affect performance. In tempered Bayes, the likelihood temperature must be chosen in advance \citep{wu2023comparison}, and this choice can substantially alter posterior behavior under model misspecification \citep{grunwald2017inconsistency}. In variational Bayes, the variational family and the temperature scale determine how empirical loss is balanced against KL regularization. In Gaussian process regression, the prior bandwidth determines the effective smoothness class of the posterior. Thus, quantities often treated as tuning parameters are in fact statistical choices: they shape posterior behavior under misspecification, nonstationarity, heavy-tailed data, and unknown smoothness.

This paper takes the view that these inferential choices should themselves be adapted online. Rather than selecting a single update rule in advance, we maintain a collection of Bayesian experts.
Each expert follows its own exact, generalized, or variational Bayesian update rule. A meta-learner then aggregates the experts' posterior predictive distributions using sequential predictive losses. The resulting predictor is a mixture over Bayesian predictive distributions: it retains the uncertainty-aware form of Bayesian prediction while gaining the adaptivity of prediction with expert advice.

This leads to a two-level view of Bayesian online learning.
At the lower level, each expert updates a posterior or variational approximation over model parameters. At the upper level, the learner updates a distribution over Bayesian update rules.
Given expert specifications $\{(\eta_k,\pi_k,\mathcal{Q}_k)\}_{k=1}^K$,
expert \(k\) maintains \(q_{t,k}\), and the aggregate predicts with $q_t = \sum_{k=1}^K \omega_{t,k} q_{t,k}$, where the weights are updated from the experts' observed predictive losses. Thus, the proposed approach does not select a single inferential rule once and for all. Instead, it forms an adaptive posterior mixture whose weights evolve with the data stream. This perspective connects Bayesian online learning with prediction with expert advice, while preserving the uncertainty-aware form of Bayesian prediction.

Our purpose is not to introduce a new abstract expert-aggregation inequality. The relevant online-learning tools are well established \citep{cesabianchi2006prediction,hazan2016introduction}. Rather, the contribution is to identify posterior aggregation as a modular mechanism for making Bayesian online learning adaptive to statistically meaningful inferential choices, and to derive its consequences in Bayesian settings where such choices are especially important. In particular, we use a general aggregation principle as a \textit{transfer} device: once individual Bayesian experts satisfy suitable single-expert guarantees, the aggregate inherits the performance of the best expert in hindsight up to a standard meta-level cost.

Our central observation is that this aggregation cost is governed by a single design choice: the scalar measure used to evaluate each expert. We consider two types of natural choices: the \emph{mean-loss}  $\E_{\theta \sim q_{t,k}}[\ell_t(\theta)]$ and  the \emph{annealed-loss} $-\gamma^{-1}\log\E_{\theta \sim q_{t,k}}[\exp(-\gamma\ell_t(\theta))]$. We establish general bounds on the aggregation cost for both evaluation methods. For the mean-loss, this bound is of order $O(\sqrt{T\log K})$, while a fast $O(\log K)$ bound is attained for the annealed loss.
 
We provide two applications, each of which corresponds to the mean-loss and annealed-loss aggregations, respectively. With the mean-loss aggregation, we study online conformal inference, where the pinball loss is used to update and aggregate experts. We show that Gaussian streaming variational update on the pinball loss gives a smoothed Bayesian counterpart of adaptive conformal inference \citep{gibbs2021adaptive}.
 
With the annealed-loss aggregation, we study online Gaussian process (GP) regression with unknown smoothness. Each expert uses a GP prior with a different bandwidth, and the aggregate mixes the posterior predictive densities. The mixture satisfies an oracle inequality in cumulative predictive KL risk and, over a dyadic bandwidth grid, adapts to unknown H\"older smoothness at the minimax rate up to logarithmic factors. Bandwidth aggregation thus plays the role of a hyperprior on the bandwidth, implemented sequentially through predictive performance; the fast $\log K$ aggregation cost is what makes the aggregation cost negligible against the statistical rate.
 
Our contributions are as follows.
\begin{itemize}
    \item  We formulate Bayesian online learning as a two-level sequential learning problem and identify the weight scoring as the design choice that governs the aggregation cost: the mean-loss score pays $O(\sqrt{T\log K})$, while the annealed-loss score pays only $O(\log K)$. These results are general as they hold for arbitrary experts. 

    \item For the mean-loss form, we derive a smoothed Bayesian counterpart of adaptive conformal inference and establish long-run randomized coverage for both single experts and their aggregate.
    
    \item For the annealed-loss form, we develop an aggregated GP predictive density that satisfies an oracle inequality in cumulative predictive KL risk and attains the minimax-adaptive rate over H\"older classes up to logarithmic factors.

    \item Experiments across online variational benchmarks, nonstationary conformal calibration, and GP regression show that the proposed aggregate tracks strong experts without oracle hyperparameter selection and is more stable than hard expert selection.
\end{itemize}

The rest of this paper is organized as follows. In \cref{sec:preliminary}, we provide some background materials and related work. In \cref{sec:method}, we introduce the proposed online Bayesian aggregation approach and discuss the choice of the evaluation metric for Bayesian experts. In \cref{sec:regret}, we provide general regret bounds for our aggregation schemes. Two applications of our approach are considered: online conformal inference in \cref{sec:conformal} and online nonparametric regression with GP in \cref{sec:regression}. We develop related statistical guarantees for both applications as well.  In \cref{sec:experiment}, we perform numerical experiments on various synthetic and real datasets. We conclude the paper in \cref{sec:conclusion}.

\section{Preliminaries}
\label{sec:preliminary}

\subsection{Notation}

For $x\in\R$, let $(x)_+:=\max\{0,x\}$, which is the positive part of $x$.  For a natural number $K$, we denote $[K]:=\{1,\dots, K\}$.  For two positive sequences $(a_n)_{n\in \mathbb{N}}$ and $(b_n)_{n\in \mathbb{N}}$, we write $a_n\lesssim b_n$ or  $b_n\gtrsim a_n$, if there exists a positive constant $C>0$ such that $a_n\le Cb_n$ for any $n\in \mathbb{N}$. Moreover, we write $a_n\asymp b_n$ if both $a_n\lesssim b_n$  and  $a_n\gtrsim b_n$ hold. Let $\phi$ and $\Phi$ be the density and cdf of the standard normal distribution, respectively. Let $\phi_\varsigma(\cdot)= \varsigma^{-1} \phi(\cdot/\varsigma)$, i.e., $\phi_\varsigma$ is the density of $\N(0,\varsigma^2)$. For a function $f$ supported on $\cX$, let $\|f\|_{\cL_2(P)}:=(\int_{\cX}|f(x)|^2\d P(x))^{1/2}$ be the $\cL_2$ norm with respect to a measure $P$ on $\cX$ and let $\|f\|_{\infty}:=\sup_{x\in\cX}|f(x)|$ be the uniform norm on $\cX$. Let $\cC(\cX)$ and $\cC^\beta(\cX)$ be the spaces of continuous and $\beta$-Hölder smooth functions on the domain $\cX$, respectively, equipped with the uniform norm.

\subsection{Bayesian online learning}
A convenient perspective on online Bayesian updating is as a regularized cumulative-risk minimization problem over distributions. Specifically, given a prior $\pi$, learning rate $\eta>0$ and per-round loss $\ell_t(\theta) := \ell(Y_t, \theta)$, we may define the (generalized) online posterior by
\begin{align}
\label{eq:posterior_update}
p_{t+1} = \argmin_{q \ll\pi} \left\{ \E_{\theta \sim q}\sbr[3]{\sum_{s=1}^t \ell_s(\theta)} + \frac{1}{\eta} \kl(q,\pi)\right\},
\end{align}
which, by the Donsker-Varadhan variational representation, turns out to be the exponential weights update
\begin{align*}
p_{t+1}(\theta)\propto \exp \bigl(-\eta \,\ell_t(\theta)\bigr)\,p_t(\theta).
\end{align*}
When the exact online posterior is intractable, a common strategy is to maintain a tractable approximation $q_t=q_{\psi_t}$ and update the variational parameter $\psi_t$ (living in the space $\Psi$) sequentially using the arriving datum $Y_t$, such as
    \begin{align*}
        \psi_{t+1} = \argmin_{ \psi \in \Psi} \left\{ \E_{\theta \sim q_{\psi}}\sbr[3]{\sum_{s=1}^t \ell_s(\theta)} + \frac{1}{\eta} \kl(q_\psi,\pi)\right\},
    \end{align*}
which restricts the domain of the optimization problem \eqref{eq:posterior_update} to a tractable family of distributions.

There is a natural connection between generalized Bayesian updating and online mirror descent: the posterior $p_{t+1}$ is the solution of a regularized cumulative loss minimization problem over probability measures, with KL divergence serving as the mirror map. Equivalently, the one-step update may be written as a mirror descent step on the linearized loss function $\psi\mapsto \bar L_t(\psi):=\E_{\theta\sim q_\psi}[\ell_t(\theta)]$, with different choices of regularization and proximity terms. Early approaches apply online gradient descent (OGD) to variational parameters , using a Euclidean proximal term to stabilize updates (see, e.g., \citet{shalev2025online}):
\begin{align}
\label{eq:oga_update}
    \psi_{t+1} =  \argmin_{ \psi \in \Psi} \left\{ \sum_{s=1}^t \bigr\langle \psi,\nabla_{ \psi} \bar L_s( \psi_s) \bigr\rangle+ \frac{1}{\eta} \| \psi - \psi_1 \|^2 \right\}
\end{align}
which is referred to as online gradient approximation (OGA).
Replacing the Euclidean prox with an information-theoretic regularizer yields sequential variational approximation (SVA) \cite{dai2016provable}, which uses the Kullback-Leibler (KL) divergence $\KL(q_\psi,\pi)$ as a penalty toward the prior:
\begin{align}
    \psi_{t+1} = \argmin_{\psi \in \Psi} \left\{ \sum_{s=1}^t \bigr\langle \psi,\nabla_{ \psi} \bar L_s( \psi_s) \bigr\rangle+ \frac{1}{\eta} \KL(q_{\psi}, \pi) \right\}
\end{align}
Streaming variational Bayes (SVB) \citep{broderick2013streaming} takes a different view, encouraging local stability by adding a proximal regularizer $\KL(q_\psi , q_{\psi_t})$ that keeps updates close to the previous iterate:
\begin{align}
\label{eq:svb_update}
    \psi_{t+1} = \argmin_{\psi \in \Psi} \left\{\bigr\langle \psi,\nabla_{ \psi} \bar L_t( \psi_t) \bigr\rangle + \frac{1}{\eta} \KL(q_{\psi}, q_{\psi_{t}}) \right\}.
\end{align}

Cumulative (regret-type) guarantees for such online Bayesian updates have been established in several forms. Online regret bounds for Bayesian posteriors under Gaussian priors are given by \citet{kakade2004online} and for sequential Gaussian process prediction, \citet{seeger2008information} established information-consistency bounds for cumulative predictive KL risk. For variational approximations specifically, \citet{cherief2019generalization} provided regret bounds for both SVA and SVB algorithms, controlling their approximation gap to the exact Gibbs posterior within the cumulative regret. \citet{alquier2021non} extended these results to general $f$-divergence regularizers beyond the KL divergence.

\subsection{Online expert aggregation}

We briefly review prediction with expert advice, which provides the meta-level machinery of our framework; see \citet{cesabianchi2006prediction} and \citet{hazan2016introduction} for comprehensive treatments. A learner has access to $K$ experts. At each round $t$, expert $k$ incurs a loss $\ell_{t,k} \in \mathbb{R}$, and the learner, who maintains a weight vector $\omega_t \in \Delta_K$, incurs the mixture loss $\sum_{k=1}^K \omega_{t,k}\ell_{t,k}$. The learner's goal is to control the regret to the best expert in hindsight,
    \begin{align*}
    R_T := \sum_{t=1}^{T} \ell_t(\text{learner}) - \min_{k\in[K]} \sum_{t=1}^{T} \ell_{t,k}.
    \end{align*}
    
The exponential-weights (Hedge) algorithm updates $\omega_{t+1,k} \propto \omega_{t,k}\exp\del{-\gamma \ell_{t,k}}$ with a meta-rate $\gamma > 0$. For losses bounded in $[0, B]$, the standard analysis controls the potential $Z_t = \sum_k \omega_{1,k}\exp\del[0]{-\gamma\sum_{s\le t}\ell_{s,k}}$ via the inequality $\e^{-x} \le 1 - x + x^2$, which leaves a quadratic residual at every round. Optimizing $\gamma \asymp \sqrt{\log K/T}$ then yields
    \begin{align*}
    R_T \lesssim B\sqrt{T\log K}.
    \end{align*}
This $O(\sqrt{T})$ rate is unimprovable in general for bounded losses, and the optimal tuning of $\gamma$ requires knowledge of the horizon $T$.

The $\sqrt{T}$ barrier disappears for losses with sufficient curvature. A loss is $\gamma$-mixable if, for any distribution $\omega \in \Delta_K$ over expert predictions, there exists an aggregated prediction whose loss is at most $-\gamma^{-1}\log\sum_k \omega_k \exp\del{-\gamma \ell_{t,k}}$ \citep{vovk1998game}. This is a weaker notion than $\gamma$-exp-concave loss functions. For $\gamma$-mixable losses, the aggregating algorithm with $\gamma = \gamma$ makes the exponential-weights potential telescope exactly, with no quadratic residual, giving the horizon-free bound
\begin{align*}
R_T \le \frac{\log K}{\gamma}.
\end{align*}

The bounds above compare against a single fixed expert. When the identity of the best expert changes over the stream, the relevant benchmark is the best sequence of experts with at most $m$ switches, and the fixed-share update \citep{herbster1998tracking,gradu2023adaptive}
\begin{align*}
\omega_{t+1,k} = (1-\sigma)\bar\omega_{t,k} + \frac{\sigma}{K}
\end{align*}
guarantees shifting regret of order $\sqrt{T m\log (KT)}$ for bounded losses and $ m\log (KT)$ for mixable losses. The sharing step prevents any weight from collapsing to zero, which is what allows recovery after a regime change.

\subsection{Further related work}

\paragraph{Online conformal inference}
Conformal prediction provides finite-sample distribution-free uncertainty quantification under exchangeability, but the standard validity guarantee is not directly designed for nonstationary data streams, which are common in online learning. Adaptive conformal inference (ACI) of \citet{gibbs2021adaptive} addresses this issue by updating the conformal threshold through an online gradient step on the pinball loss, yielding long-run coverage guarantees under arbitrary distribution shift. Subsequent work removes the step-size sensitivity of ACI by aggregating a grid of learning rates, such as AgACI \citep{zaffran2022adaptive}, DtACI \citep{gibbs2024conformal}  and SAOCP \citep{bhatnagar2023improved}.

\paragraph{Gaussian process adaptation} 
Gaussian process priors are widely used for nonparametric regression, and their frequentist performance depends strongly on the scaling or bandwidth parameter. In the Bayesian nonparametric literature, adaptation to unknown smoothness has been studied through hierarchical or empirical Bayes constructions that put a prior on the bandwidth or scaling parameter; see, for example, \citet{de2010adaptive,van2009adaptive,rousseau2017asymptotic}. Our GP result is motivated by this line of work, but differs in its online and predictive formulation. Instead of constructing a fully hierarchical GP prior and analyzing posterior contraction in batch form, we run a finite grid of GP experts with different bandwidths and aggregate their posterior predictive densities sequentially, leading to adaptation to unknown H\"older smoothness.

\section{Expert aggregation for Bayesian online learning}
\label{sec:method}

\subsection{Bayesian expert aggregation algorithm}

Fix a grid of specifications $\cbr{(\eta_k, \pi_k, \cQ_k) : k \in [K]}$, where $\eta_k > 0$ is a learning rate, $\pi_k$ a prior, and $\cQ_k$ a variational family. For each $k \in [K]$, expert $k$ runs an online Bayesian or variational update and maintains a distribution $q_{t,k} \in \cQ_k$ at the start of round $t$, formed from $\cD_{t-1}:=(Y_s)_{s=1}^{t-1}$. Upon observing $D_t$ it incurs the nonnegative per-round loss $\ell_t(\theta) := \ell(Y_t, \theta)$. 

A meta-learner maintains weights $\omega_t = (\omega_{t,1}, \dots,\omega_{t,K}) \in \Delta_K$ and predicts with the mixture
    \begin{align*}
        q_t(\theta) := \sum_{k=1}^K \omega_{t,k}q_{t,k}(\theta).
    \end{align*}
The weights are updated from the experts' observed losses by an exponential-weights rule with meta-rate $\gamma > 0$, followed by an optional fixed-share step with parameter $\sigma \in [0, 1]$ that prevents any weight from collapsing \citep{hazan2016introduction}:
\begin{align}
  \label{eq:weight-update1}
  \bar\omega_{t,k} &\propto \omega_{t,k}\exp\del[1]{-\gamma\mathsf{loss}_{t,k}},\\
  \omega_{t+1,k} &= (1 - \sigma)\bar\omega_{t,k} + \frac{\sigma}{K},
  \label{eq:weight-update2}
\end{align}
where $\mathsf{loss}_{t,k}$ is a scalar summarizing how well \textit{distributional} expert $q_{t,k}$ performed at round $t$. We consider the two choices, given a selected loss function $\ell_t$. The proposed algorithm is summarized in \cref{alg:ova}.

\begin{algorithm}[H]
\caption{Online Bayesian Expert Aggregation}
\label{alg:ova}
\begin{algorithmic}
\State \textbf{Input:}
grid of learning rates, prior distributions and variational families $\{(\eta_k, \Pi_k, \cQ_k)\}_{k=1}^K$; meta rate $\gamma>0$; share $\sigma\in(0,1]$; $\omega_1\in\Delta_K$.
\State \textbf{Initialize:} $q_{1,k}\gets \pi_k$ for all $k\in[K]$.
\For{$t=1,2,\ldots,T$}
  \State \textbf{Return} $q_{t}=\sum_{k=1}^K \omega_{t,k} q_{t,k}$
  \State Observe the data point $Y_t$ and compute the loss scoring $\mathsf{loss}_{t,k}.$
  \State Compute the weight vector $(\bar\omega_{t,1},\dots, \bar\omega_{t,K})$ as $\bar\omega_{t,k}\propto \omega_{t,k}\exp\del[1]{-\gamma \mathsf{loss}_{t,k}}$
  \State Update $\omega_{t+1,k}\gets (1-\sigma)\bar\omega_{t,k}+\sigma/K$ 
  \For{$k=1,\ldots,K$}
    \State Update
    the posterior distribution $q_{t+1,k}$ from $q_{t,k}$ by performing one of OBI algorithms.
  \EndFor
\EndFor
\end{algorithmic}
\end{algorithm}

\subsection{Two types of loss for Bayesian experts}

For frequentist or point experts, say $\hat\theta_{t,k}$, the loss is naturally defined as $\ell_t(\hat\theta_{t,k})$. For Bayesian experts $q_{t,k}$,  we consider two different ways to define the loss.

\paragraph{Mean-loss} The most direct choice is the expected loss with respect to an expert $q$,
    \begin{align}
     L_{t}(q) := \E_{\theta \sim q}\sbr{\ell_t(\theta)}.
    \end{align}
It depends on the expert only through the \textit{first moment} of the loss.

\paragraph{Annealed-loss}
The alternative is the annealed-loss (terminology from statistical mechanics)
    \begin{align}
    \Lambda_{t}(q)  := -\frac{1}{\gamma}\log\del[2]{ \E_{\theta \sim q}\sbr[1]{\exp\del[1]{-\gamma\ell_t(\theta)}}},
    \end{align}
which is the cumulant generating function $\ell_t(\theta)$ under $q$; its expansion $\Lambda_{t}(q) =\E_{\theta \sim q}\sbr{\ell_t(\theta)}-\frac{\gamma}{2}\mathsf{Var}_{\theta\sim q}(\ell_t(\theta))+O(\gamma^2)$ shows that it reflects not only the mean loss but the full spread of $\ell_t(\theta)$ across the posterior. The mean-loss scoring is recovered as the $\gamma\to0$ limit retaining only the first cumulant. 

The annealed-loss acquires a direct statistical meaning under the log-density loss. Suppose $\ell_t(\theta) = -\log \p(Y_t|\theta)$ and $\gamma= 1$. Then for any expert $q_{t,k}$, we have
\begin{align*}
    \exp\del[1]{-\Lambda_t(q_{t,k})}
    =\p_{t,k}(Y_t|\cD_{t-1})
    := \E_{\theta\sim q_{t,k}}\sbr[1]{\p(Y_t|\theta)},
\end{align*}
so the annealed loss is precisely the predictive log-loss: the negative log-density that expert $k$'s posterior predictive distribution assigns to the observation it is about to see. Scoring an expert by $\Lambda_t$ thus evaluates its entire predictive distribution, formed before $Y_t$ arrives. This interpretation turns the regret benchmark into a predictive risk. Suppose
that $Y_t|\cD_{t-1} \sim \p^\star_t$. The excess cumulative loss over the data generator is the cumulative log-likelihood ratio,
\begin{align*}
    \sum_{t=1}^T \Lambda_t(q_{t,k}) - \cbr[1]{-\log \p^\star_t(Y_t)}
    = \sum_{t=1}^T \log\frac{\p^\star_t(Y_t)}{\p_{t,k}(Y_t|\cD_{t-1})}.
\end{align*}
Since $\p_{t,k}(\cdot|\cD_{t-1})$ is $\cD_{t-1}$-measurable, taking expectations and conditioning on $\cD_{t-1}$ in each summand gives
\begin{align*}
    \sum_{t=1}^T \E_{\cD_t}\sbr[1]{\Lambda_t(q_{t,k}) - \cbr[1]{-\log \p^\star_t(Y_t)}}
    = \sum_{t=1}^T \E_{\cD_{t-1}}\sbr[1]{\kl\del[1]{\p^\star_t, \p_{t,k}(\cdot|\cD_{t-1})}}.
\end{align*}
In other words, the expected excess annealed loss is the cumulative predictive KL risk.

\section{Regret bounds for Bayesian expert aggregation}
\label{sec:regret}

\subsection{Regret bound for mean-loss aggregation}
\label{sec:regret_mean}

We define the mean-loss aggregation regret with respect to the best expert on the time interval $\cI\subset[T]$ as
    \begin{align}
        \cR_{L}(\cI):= \sum_{t\in\cI}L_t(q_t) - \inf_{k\in[K]} \sum_{t\in\cI}L_t(q_{t,k}).
    \end{align}
 
\begin{theorem}[Regret bound; mean-loss aggregation]
\label{thm:regret_mean}
Let $\sigma\in(0, 1/2]$.Then for any interval $\cI\subset[T]$,
    \begin{align}
      \cR_{L}(\cI)\le \gamma \sum_{t\in \cI} \bar{L^2}(t) 
      +\frac{1}{\gamma}\del[1]{\log(K/\sigma) + 2\abs{\cI}\sigma},
    \end{align}
where  $\bar{L^2}(t):=\sum_{k=1}^K\omega_{t,k}\{L_t(q_{t,k})\}^2$. If we take 
    \begin{align*}
        \gamma \asymp \sqrt{\frac{\log(|\cI|K)}{ \sum_{t\in\cI}\bar{L^2}(t)}}, \quad \sigma \asymp 1/ |\cI|,
    \end{align*}
we have the bound
 \begin{align}
      \cR_{L}(\cI)\lesssim \sqrt{ \log(K|\cI|)\sum_{t\in\cI}\bar{L^2}(t)}.
    \end{align}
Moreover, when the loss is bounded as $\ell_t(\theta)\le B$ for any $\theta$ and $t$, then if we take 
$\gamma \asymp \sqrt{\frac{\log(|\cI|K)}{|\cI|B^2}}$, $\sigma \asymp 1/ |\cI|$, we have the much simpler bound
    \begin{align}
      \cR_{L}(\cI)\lesssim \sqrt{ |\cI|\log(K|\cI|)}.
    \end{align}
\end{theorem}

In \cref{thm:regret_mean}, the mean-loss aggregation pays the slow regret bound. Technically, this is due to controlling the potential $\sum_{k=1}^K \omega_{1,k}\exp(-\sum_{s<t}L_s(q_{s,k}))$ requires the inequality $\exp(-x) \le 1 - x + x^2$, which leaves a quadratic residual, yielding the quadratic term  $\bar{L^2}(t)$.

\subsection{Fast regret bound for annealed-loss aggregation}
\label{sec:regret_annealed}

We now turn to the annealed-loss $\Lambda_{t}(q) = -\gamma^{-1}\log\E_{\theta\sim q}[\exp(-\gamma\ell_t(\theta))]$. We define the annealed-loss aggregation regret with respect to the best expert on the time interval $\cI\subset[T]$ as
    \begin{align}
        \cR_{\Lambda}(\cI):=  \sum_{t\in\cI}\Lambda_t(q_t) - \inf_{k\in[K]} \sum_{t\in\cI}\Lambda_t(q_{t,k}).
    \end{align}

\begin{theorem}[Regret bound; annealed-loss aggregation]
\label{thm:regret_annealed}
Let $\sigma\in(0, 1/2]$. Then for any interval $\cI\subset[T]$,
    \begin{align}
       \cR_{\Lambda}(\cI)\le 
       \frac{1}{\gamma}\del[1]{\log(K/\sigma) + 2\abs{\cI}\sigma}.
    \end{align}
When we choose $\sigma\asymp 1/|\cI|$, we have
    \begin{align}
        \cR_{\Lambda}(\cI)\lesssim  \log(K|\cI|).
    \end{align}
\end{theorem}

The regret is $O(\log (K|\cI|))$, only logarithmically depending on the interval length $|\cI|$. No horizon-dependent tuning of $\gamma$ is required. 

The aggregation algorithm with no share, i.e., $\sigma=0$, attains the fast regret bound also.

\begin{theorem}[Regret bound; no share]
\label{thm:regret_noshare}
Let $\sigma=0$. Then
    \begin{align}
        \cR_{\Lambda}([T])\le 
        \frac{1}{\gamma} \log(1/\omega_{1, k^\star}),
    \end{align}
where $k^\star \in \argmin_{k \in [K]} \sum_{t=1}^T \Lambda_t(q_{t,k})$.
In particular, when we  use the uniform initial weight $\omega_{1,k}=1/K$,
    \begin{align}
        \cR_{\Lambda}([T])\le \frac{1}{\gamma}\log K.
    \end{align}
\end{theorem}

\subsection{Discussion}

The constant aggregation regret in \cref{sec:regret_annealed} has the same origin as the fast rates that exp-concave losses enjoy in prediction with expert advice. For any loss function, the exponential weighting can bound  the cumulative \emph{mix loss} $M_t(\hat\theta_t)=-\gamma^{-1}\log\sum_k \omega_{t,k}\exp\del[0]{-\gamma \ell_t(\hat\theta_{t,k})}$ for an aggregate $\hat\theta_t=\sum_{k}\omega_{t,k}\hat\theta_{t,k}$ of frequentist experts $\hat\theta_{t,k}$, by the best expert's cumulative loss plus $\gamma^{-1}\log K$. Fast and slow rates are separated solely by the per-round gap between the loss the learner actually incurs and the mix loss. Closing this gap is precisely the definition of $\gamma$-exp-concavity of $\ell_t$ in $\theta$, a genuine restriction that the log loss satisfies but piecewise-linear losses such as the pinball loss do not; absent it, the gap is controlled by $\e^{-x} \le 1 - x + x^2$ and the quadratic residual forces the $\sqrt{T}$ rate. The annealed-loss regret is fast because the annealed loss makes the gap vanish identically: $q \mapsto \exp\del{-\gamma\Lambda_t(q)} = \E_{\theta\sim q}\sbr{\exp(-\gamma\ell_t(\theta))}$ is \textit{affine} in $q$, so the mixture of Bayesian experts attains the mix loss with equality, for an arbitrary underlying loss function $\ell_t$. The lift to distribution space thus manufactures the curvature that $\ell_t$ may lack in the parameter space; the curvature comes not from the geometry of the loss but from the exponential sitting inside the expectation. By contrast, the mean-loss places the expectation outside the exponential, so the mixture's loss is the arithmetic mean of the experts' losses while the mix loss is a soft minimum, and the per-round gap, of order $\gamma$ times the spread of the losses under $\omega_t$, is exactly the quadratic residual above.  

This contrast might suggest that the mean-loss is obsolete: the annealed-loss is defined for any loss function, so one could always aggregate at the fast rate. It does not: the fast rate is a guarantee on the annealed regret, which controls the cumulative predictive loss but not mean functionals, as $\Lambda_t(q)\le L_t(q) $ for any $q$ by Jensen's inequality. For example, in our conformal application given in the next section, the quantity of interest is the expected pinball loss, which is a mean functional, so the mean-loss is the one whose regret speaks to the target. The choice between the two losses is therefore not a choice between a slow and a fast bound for the same goal, but a matching of the scoring to the criterion: the annealed-loss aggregation when the target is the cumulative predictive loss, at the fast rate; the mean-loss aggregation when the target is a posterior mean functional, at the slow rate that this weaker criterion genuinely costs.

\subsection{Point prediction via posterior means}

The guarantees of \cref{sec:regret_mean,sec:regret_annealed} are stated for the scores $L_t$ and
$\Lambda_t$ of the mixture $q_t$.  In this section, we show that these can be transferred to a performance guarantee for the posterior mean $\bar\theta_t := \E_{\theta\sim q_t}[\theta]$. Since the mixture's mean is the weighted average of the experts' means, i.e., 
    \begin{align*}
    \bar\theta_t = \sum_{k=1}^K \omega_{t,k}\bar\theta_{t,k},
    \quad \bar\theta_{t,k} := \E_{\theta\sim q_{t,k}}[\theta],
    \end{align*}
predicting with $\bar\theta_t$ is itself an instance of prediction with expert advice with point experts $\bar\theta_{t,k}$. The result of this section reproduces, at the level of Bayesian experts, the classical dichotomy of online learning in which the curvature of the loss separates the slow and fast regimes: convexity suffices to transfer the mean-loss bound, while
exp-concavity transfers the annealed-loss bound.

\begin{itemize}
    \item Suppose $\theta\mapsto\ell_t(\theta)$ is convex for every $t$. Then Jensen's inequality gives 
        \begin{align*}
            \ell_t(\bar\theta_t)\le\E_{\theta\sim q_t}\sbr{\ell_t(\theta)} = L_t(q_t)
        \end{align*}
    for every $t$. Hence, \cref{thm:regret_mean} gives a slow upper bound of the cumulative loss $\sum_{t\in\cI} \ell_t(\bar\theta_t)$ of the posterior mean.

    \item Suppose $\theta\mapsto\ell_t(\theta)$ is $\gamma$-exp-concave for every $t$, i.e., $\theta\mapsto\exp(-\gamma\ell_t(\theta))$ is concave. Then Jensen's inequality applied to $q_t$ gives
    \begin{align*}
        \exp(-\gamma\ell_t(\bar\theta_t)) \ge \E_{\theta\sim q_t}\sbr{\exp(-\gamma\ell_t(\theta))}
            = \exp(-\gamma\Lambda_t(q_t)),
    \end{align*}    
    that is,
    \begin{align*}
        \ell_t(\bar\theta_t)\le\Lambda_t(q_t).
    \end{align*}
    Hence, \cref{thm:regret_annealed} gives a fast upper bound of the cumulative loss.
\end{itemize}

These two observations delimit precisely when the mean-loss aggregation is genuinely needed for point prediction: for losses that are convex but not exp-concave. The hinge and pinball loss functions are the canonical examples; being piecewise linear, they admit no exp-concavity at any $\gamma>0$, so the fast bound is unavailable, and the slow $\sqrt{T}$ rate is the price of this weaker curvature. By contrast, when the loss function itself is exp-concave, the annealed-loss aggregation yields the fast rate even though the target is a point predictor.

\section{Application to online conformal inference}
\label{sec:conformal}

In this section, we apply the mean-loss aggregation of Bayesian experts to online conformal inference, which aims to construct prediction sets for the outcome variable with valid long-run coverage under distribution shift. Suppose that at each time $t$,  a pair of an input and output variables $(X_t, Y_t)\in \mathcal X \times \mathcal Y$ arrives. For any threshold $\theta \ge 0$, define the prediction set
\begin{align*}
    \hat C_t(\theta):= \{ y \in \mathcal{Y} : S_t(X_t, y) \le \theta \},
\end{align*} where $S_t : \mathcal X \times \mathcal Y \to [0, R]$ is a bounded conformal score function.
After observing $Y_t$, define the minimal covering threshold as
\begin{align*}
    r_t := \inf\{\theta\ge 0:\ Y_t\in \hat C_t(\theta)\}.
\end{align*}
For the target miscoverage level $\alpha\in(0,1)$, we define the pinball loss
    \begin{align*}
    \ell_{\alpha}(\theta,r)
   & :=  
    \begin{cases}
    \alpha(\theta-r), & \theta\ge r,\\
    (1-\alpha)(r-\theta), & \theta< r
    \end{cases}\\
   & = \alpha(\theta - r)_+ + (1 - \alpha)(r - \theta)_+
\end{align*}
and use this loss function for updating and evaluating experts. For an expert $k$, we consider a variational distribution of the form $\N(\psi,\tau_k^2)$ with a pre-specified standard deviation $\tau_k>0$ and update its mean parameter $\psi$ using the streaming variational Bayes (SVB) algorithm. We write the expected pinball loss under the variational distribution $\N(\psi,\tau_k^2)$ as
    \begin{align}
    \label{eq:svb_aci_update}
        \bar{L}_{t,k}(\psi) := \mathbb{E}_{\theta \sim \N(\psi,\tau_k^2)}[\ell_\alpha(\theta, r_t)].
    \end{align}
We perform the SVB update with per-step KL proximal term:
    \begin{align}
    \label{eq:svb_aci}
          \psi_{t+1,k} = \argmin_{\psi\in \R}
  \Bigl\{\bigl\langle  \frac{\d}{\d\psi} \bar{L}_{t,k}(\psi_{t,k}),\psi \bigr\rangle
    + \frac{1}{\eta_k} \mathrm{KL}(\N(\psi,\tau_k^2), \N(\psi_{t,k}, \tau_k^2))
  \Bigr\},
    \end{align}  
where $\eta_k>0$ is the learning rate of the expert $k$. Then the expert $k$ constructs a prediction interval for the next datum $Y_{t+1}$ based on the variational posterior $q_{t+1,k}=\N(\psi_{t+1,k}, \tau_k^2)$.  The SVB update \eqref{eq:svb_aci_update} has a closed-form expression as given in the next proposition.

\begin{proposition}[One-step mean update]
\label[proposition]{lem:cp_svb_one_step}
The SVB update \eqref{eq:svb_aci}  is written as
    \begin{align}
    \label{eq:svb_aci_update}
        \psi_{t+1,k}
        = \psi_{t,k}
    - \eta_k \tau_k^2 \sbr[3]{\alpha - \Phi\del[2]{\frac{r_t - \psi_{t,k}}{\tau_k}} }.
    \end{align}
\end{proposition}

Note that when $\tau_k \to 0$ and $\eta_k\tau_k^2\to\eta$, we have $ \Phi\del[2]{\frac{r_t - \psi_{t,k}}{\tau_k}}\to \ind(r_t > \psi_{t,k}) $, thus  \eqref{eq:svb_aci_update} recovers the standard ACI update rule \citep{gibbs2021adaptive}
        \begin{align*}
       \psi_{t+1,k}=\psi_{t,k}-\eta(\alpha - \ind(r_t > \psi_{t,k})).
    \end{align*}
Thus, the SVB update can be interpreted as a Bayesian-smoothed version of the ACI update rule. Accordingly, we refer to the SVB online conformal inference procedure as Bayes-ACI.

We aggregate multiple Bayes-ACI experts over a certain grid $\{(\eta_k, \tau_k): k\in[K]\}$ using the mean-loss aggregation algorithm. Specifically, the weights are updated as 
    \begin{align*}
          \bar\omega_{t,k} &\propto \omega_{t,k}\exp\del[1]{-\gamma \bar{L}_{t,k}(\psi_{t,k})},\\
  \omega_{t+1,k} &= (1 - \sigma)\bar\omega_{t,k} + \sigma/K
    \end{align*}
and the aggregated variational posterior $q_t:=\sum_{k=1}^K \omega_{t,k}q_{t,k}=\sum_{k=1}^K \omega_{t,k}\N(\psi_{t,k}, \tau_k^2)$ is used for online conformal inference. Likewise, this can be viewed as a Bayesian-smoothed analog of DtACI \cite{gibbs2024conformal}, which addresses the sensitivity of ACI to the step size $\eta$ by aggregating an expert grid over $\eta$. Thus, we refer to this aggregation approach as Bayes-DtACI.

By \cref{thm:regret_mean},  Bayes-DtACI with 
 \begin{align*}
        \gamma \asymp \sqrt{\frac{\log(|\cI|K)}{ \sum_{t\in\cI}\sum_{k=1}^K\omega_{t,k}\E_{\theta\sim q_{t,k}}[\ell_\alpha(\theta,r_t)]^2}}, \quad \sigma \asymp 1/ |\cI|,
    \end{align*}
attains the regret bound
    \begin{align}
          \sum_{t\in\cI}\E_{\theta\sim q_t}&[\ell_\alpha(\theta,r_t)]
       - \min_{k\in[K]} \sum_{t\in\cI}\E_{\theta\sim q_{t,k}}[\ell_\alpha(\theta,r_t)]\notag\\
       &\lesssim \sqrt{\log( |\cI|K)\sum_{t\in\cI}\sum_{k=1}^K\omega_{t,k}\E_{\theta\sim q_{t,k}}[\ell_\alpha(\theta,r_t)]^2}.
    \end{align}

Moreover,  our Bayes-ACI and Bayes-DtACI enjoy the long-term coverage guarantee.

\begin{theorem}[Long-run coverage of Bayes-ACI]
Assume that $r_t \in [0, R]$ for all time $t$ and $\psi_{1,k}\in[-\eta_k\tau_k^2-\tau_k\Phi^{-1}(\alpha),  R+\eta_k\tau_k^2-\tau_k\Phi^{-1}(\alpha)]$. Then, the long-run coverage of the $k$-th Bayes-ACI expert satisfies
\label{thm:long_run_single}
    \begin{align}
        \left|
    \frac{1}{T} \sum_{t=1}^{T}
    \mathbb{E}_{\theta \sim q_{t,k}}\bigl[\ind(r_t > \theta)\bigr] - \alpha
  \right|
  \leq \del{\frac{R}{\eta_k \tau_k^2}+2}\frac{1}{T}.
    \end{align}
\end{theorem}

\begin{theorem}[Long-run coverage of Bayes-DtACI]
\label{thm:long_run_agg}
Let $A_{k} := \eta_k \tau_k^2$, $A_{\min} := \min_k A_k$, $A_{\max} := \max_k A_k$, $\tau_{\max}:=\max_k\tau_k$, $\tau_{\min} := \min_k \tau_k$ and $C_{\alpha} := \tau_{\max}  + R + 2 A_{\max} + (\tau_{\max} - \tau_{\min})| \Phi^{-1}(\alpha)|$. Consider a modified version of Bayes-DtACI, where the parameters $\gamma$ and $\sigma$  are replaced by values $\gamma_t$ and $\sigma_t$ at time $t$. Then under the same assumption of \cref{thm:long_run_single},  the long-run coverage of Bayes-DtACI satisfies
\begin{align}
&\left|\frac{1}{T} \sum_{t=1}^{T}
\mathbb{E}_{\theta \sim q_t}\bigl[\ind(r_t > \theta)\bigr] - \alpha
\right| \notag \\
&\le 
\frac{R + 2 A_{\min}}{T A_{\min}} + \frac{ (R + A_{\max})C_{\alpha}}{A_{\min}}\frac{1}{T} \sum_{t=1}^T \gamma_t \exp({\gamma_t C_{\alpha}}) + 2\frac{R+ A_{\max}}{A_{\min}} \frac{1}{T} \sum_{t=1}^T \sigma_t .
\end{align}
In particular, if both  $\gamma_t$ and $\sigma_t$ tend to zero as $t\to\infty$, then 
    \begin{align*}
        \lim_{T \to \infty} \frac{1}{T}\sum_{t=1}^T \E_{\theta\sim q_t}[\ind(r_t > \theta)] = \alpha.
    \end{align*}
\end{theorem}

\section{Application to online Gaussian processes}
\label{sec:regression}

In this section, we apply the annealed-loss aggregation method in the setting of online nonparametric regression, with the goal of addressing a natural question: does the proposed aggregation procedure achieve the same type of adaptation that Bayesian nonparametric approaches attain through carefully designed hierarchical priors? We illustrate this with a concrete example of Gaussian process (GP) regression, where we run multiple experts with different bandwidth values. Aggregating GP experts over a bandwidth grid plays the role of a hyperprior on the bandwidth, and we show that the resulting mixture adapts to unknown H\"older smoothness at the minimax rate up to logarithmic factors. This matches the adaptation achieved by a fully hierarchical Bayesian construction, but is obtained here purely through sequential expert aggregation.

Suppose that we observe a stream $\{(X_t, Y_t)\}_{t=1}^T$ generated i.i.d. from
\begin{equation}\label{eq:model}
Y_t = f_\star(X_t) + \epsilon_t, \qquad \epsilon_t \overset{\text{i.i.d.}}{\sim} \N(0, \varsigma^2), \qquad X_t \overset{\text{i.i.d.}}{\sim} P_X,
\end{equation}
on $\cX = [0,1]^d$, where $f_\star \in \cC^\beta(\cX)$ is the unknown regression function with $\norm{f_\star}_\infty \le B$, and the smoothness $\beta > 0$ is unknown. For an inverse bandwidth $a>0$, we consider the rescaled Gaussian process prior $\pi_a$ on $C(\cX)$, which is the law of the centered Gaussian process with covariance kernel
    \begin{align*}
        c_a(x, x') = \exp \del[1]{-a^2 \norm{x - x'}^2}.
    \end{align*}
Let $\cA=\{a_1,\ldots,a_K\}$ be a inverse bandwidth grid. For each $a_k\in\cA$, the posterior at time $t$ with prior $\pi_{a_k}$ is the standard Gaussian process posterior given $\cD_{t-1} = \{(X_s, Y_s)\}_{s=1}^{t-1}$:
    \begin{equation*}
     \rho_{t,k}(f) \propto \pi_{a_k}(f)  \prod_{s=1}^{t-1} \phi_\varsigma (Y_s - f(X_s)).
    \end{equation*}
The resulting  posterior predictive density at time $t$ is given by
    \begin{align}\label{eq:gp-density}
        \p_{t,k}(y | x, \cD_{t-1})  =\E_{f\sim \rho_{t,k}}[\phi_\varsigma (y - f(x))].
    \end{align}
The next theorem gives the convergence rate of the KL divergence from the true conditional density $\p_\star(\cdot|x)=\phi_\varsigma(\cdot-f_\star(x))$ for a single GP expert. This is an instance of the information-consistency bounds for GP prediction of \cite{seeger2008information}. 

\begin{theorem}[Convergence rate of single GP]
\label{thm:gp_single}
Assume the model \eqref{eq:model} with $\norm{f_\star}_\infty \le B$ and $f_\star \in \cC^\beta([0,1]^d)$. Then the posterior predictive density \eqref{eq:gp-density} with rescaled GP prior of inverse bandwidth $a_k$ satisfies
\begin{equation}
	\frac{1}{T}	\sum_{t=1}^T\E_{\cD_t}\sbr{\kl\del[1]{
	\p_\star(\cdot|X_t),	\p_{t,k}(\cdot|X_t,\cD_{t-1})}} 
    \lesssim   \varepsilon_{T,k}^2 \log^{1+d} (T),
\end{equation}
 where we define
\begin{equation}\label{eq:rate}
    \varepsilon_{T,k}^2 : =  T^{-1} a_k^d + a_k^{-2\beta}
\end{equation}
and the implicit constant depends only on $\varsigma^2$, $B$, $\beta$, and $d$. 
\end{theorem}

\cref{thm:gp_single} shows a clear bias-variance trade-off in terms of the bandwidth: for large inverse bandwidth, the bias $a_k^{-2\beta}$ would be smaller but the estimation variance $T^{-1} a_k^d$ would be larger, and vice versa. But to optimize this trade-off, we are required to know the smoothness $\beta$, but this is rarely the case in practice. We now show that our Bayesian expert aggregation enables us to address this problem. Specifically,  we use the annealed-loss with the negative log-density loss, which is given by $ \Lambda_t(\rho_{t,k})=-\log\p_{t,k}(y | x, \cD_{t-1}) $ for aggregation. Then the weight update in \cref{alg:ova} with $\gamma=1$ and  $\sigma=0$ gives
    \begin{align*}
        \omega_{t+1,k}\propto \omega_{t,k}\exp(- \Lambda_t(\rho_{t,k}))
        = \omega_{t,k} \p_{t,k}(Y_{t}|X_{t},\cD_{t-1})
    \end{align*}
which leads to
\begin{equation*}
	\omega_{t,k}
	=	\frac{	\omega_{1,k}\prod_{s=1}^{t-1}	\p_{s,k}(Y_s|X_s,\cD_{s-1})	}
    {	\sum_{h=1}^K\omega_{1,h}\prod_{s=1}^{t-1}	\p_{s,h}(Y_s|X_s,\cD_{s-1})	}.
\end{equation*}
The aggregated posterior predictive density is
\begin{equation}\label{eq:aggregated-density}
	\p_t(y|x,\cD_{t-1})
	=\sum_{k=1}^K\omega_{t,k}\p_{t,k}(y|x,\cD_{t-1}).
\end{equation}

The following oracle KL risk bound is immediate from our general regret bound in \cref{thm:regret_noshare}, which implies that our aggregated GP predictive density is comparable to the best single experts.

\begin{corollary}[Oracle KL risk bound]
\label[corollary]{lem:log-loss-oracle}
It follows that
\begin{align}
	&\sum_{t=1}^T\E_{\cD_t}\sbr{\kl\del[1]{	\p_\star(\cdot|X_t),	\p_t(\cdot|X_t,\cD_{t-1})	}
	} \notag \\
	&\qquad\le\min_{k\in[K]}\cbr{\sum_{t=1}^T\E_{\cD_t}\sbr{\kl\del[1]{	\p_\star(\cdot|X_t),
	\p_{t,k}(\cdot|X_t,\cD_{t-1})}}	+\log(1/\omega_{1,k})	}.
	\label{eq:log-loss-oracle}
\end{align}
\end{corollary}

We take the dyadic bandwidth grid $a_k = 2^k$ for $k = 1, \ldots, K$ with $K = \lceil \log_2 T \rceil$. For any $s \in [\underline{s}, \overline{s}]$, the optimal bandwidth $a_T^* \asymp T^{1/(2s+d)}$ lies within a factor of $2$ of some grid point $a_{k_\star}$, so
\begin{equation*}
     a_{k^\star}^{-2s}\asymp T^{-1}a_{k^\star}^d  \asymp  T^{-2\beta/(2\beta+d)}.
\end{equation*}
When we use the uniform prior $\omega_{1,k}=K^{-1}$, we have by \cref{thm:regret_noshare} that the aggregation regret is $ \log K=O(\log\log T)$, which is dominated by the term $T^{-2\beta/(2\beta+d)}$ for any $\beta < \infty$. Therefore, we arrive at the following adaptive bound.
 
\begin{theorem}[Adaptive convergence rate of aggregated GP]\label{thm:adaptive}
Under the same assumption of \cref{thm:gp_single}, the aggregated GP predictive density \eqref{eq:aggregated-density} over $K=\lceil\log_2T\rceil$ many experts with the dyadic bandwidth grid $a_k=2^k$ and the uniform prior weight $\omega_{1,k}=K^{-1}$ for $k\in[K]$  satisfies
\begin{equation}
	\frac{1}{T}	\sum_{t=1}^T\E_{\cD_t}\sbr{\kl\del[1]{
	\p_\star(\cdot|X_t),	\p_t(\cdot|X_t,\cD_{t-1})	}	}
	\lesssim	 T^{-2\beta/(2\beta+d)}\log^{1+d} (T),
\end{equation}
where the implicit constant depends on $\varsigma^2$, $B$, $\beta$, and $d$.
\end{theorem}

The rate $T^{-2\beta/(2\beta+d)}$ matches the minimax rate over $\cC^\beta(\cX)$ up to a polylogarithmic factor in $T$, but the risk criterion here is the average predictive KL risk of the aggregated density rather than the $\cL_2(P_X)$ risk of an aggregated point predictor.

\begin{remark}[Unknown noise level]
In our theoretical analysis, we assume that the noise level $\varsigma$ is known to us, which is, however, not a practical assumption. To circumvent this issue, our experiments aggregate over a product grid $\cbr[0]{(a_k, \varsigma_k)}$. This is covered by the theory after lifting the parameter: let $\theta = (f, \varsigma)$ with common loss $\ell_t(f, \varsigma) = -\log\phi_\varsigma\del{Y_t - f(X_t)}$, and assign expert $k$ the prior $\pi_k = \pi_{a_k}\otimes\delta_{\varsigma_k}$. The induced predictive density is $\p_{t,k}(y|x, D_{t-1}) = \E_{f\sim\rho_{t,k}}\sbr[0]{\phi_{\varsigma_k}\del{y - f(x)}}$, and \cref{lem:log-loss-oracle}  applies verbatim since the affinity of $q\mapsto\E_{\theta\sim q}\sbr[0]{\e^{-\gamma\ell_t(\theta)}}$ is preserved on the lifted parameter space. This is the same device already implicit in the bandwidth: each expert places a point mass on its hyperparameters, and adaptation occurs at the aggregation level. 
\end{remark}

\section{Numerical experiments}
\label{sec:experiment}

\subsection{Online prediction for benchmark datasets}  
\label{sec:benchmark}

We adapt the online variational-learning benchmarks of \citet{cherief2019generalization} and \citet{haddouche2022online} to evaluate the proposed expert aggregation methods. 

\paragraph{Losses}
For binary classification, we observe $x_t,y_t$ with $x_t\in\mathbb{R}^d$ and $y_t\in\{-1,+1\}$ at round $t$. Given a parameter vector $\theta$, the learner predicts using the linear prediction $x_t^\top\theta$ and incurs the hinge loss $\ell_t(\theta)=(1-y_t x_t^\top\theta)_+$. For regression,  we observe $x_t,y_t$ with $y_t\in\mathbb{R}$ at round $t$, and the loss is $\ell_t(\theta) = (y_t - \theta^{\top} x_t)^2$.

\paragraph{Variational family}
To approximate the one-step Bayesian update, we use the mean-field Gaussian family
\begin{align*}
    \cQ
    =
    \{q_\psi=\N(m, \mathrm{diag}(\sigma^2)):
   \psi=(m,\sigma),\ m\in \mathbb R^d,\ \sigma\in \mathbb R^d_+\}.
\end{align*}
Predictions are made using the posterior mean.

\paragraph{Algorithms}
We compare expert-aggregation versions of SVB, OGA,and OGD, denoted by SVB-EA, OGA-EA, and OGD-EA, respectively. OGA and SVB are Bayesian online procedures that maintain a variational distribution over the parameter: OGA updates the variational parameter through a linearized objective with a Euclidean proximal regularizer as in \cref{eq:oga_update},  whereas SVB performs a one-step variational Bayes update with a local KL proximal regularizer as in \eqref{eq:svb_update}. We provide the detailed update formula for OGA and SVB in \cref{appendix:oga_svb}. The OGD update is given by 
\begin{align*}
    m_{t+1,k} \gets m_{t,k} - \eta \nabla_m \ell_t(m_{t,k}).
\end{align*}
which is included as a frequentist counterpart. It maintains only a point estimate, and hence does not produce a posterior distribution.

For SVB-EA and OGA-EA, we use mean-loss to aggregate our experts.
Since OGD-EA does not produce a posterior distribution to score, its aggregation weights are updated by the standard fixed-share rule using the point-prediction loss of each expert:
$$
w_{t+1,k} \propto w_{t,k} \exp \left( -\gamma_t \ell_t(m_{t,k}) \right)
$$
The base experts follow the step-size schedules of \citet{cherief2019generalization}. 
Let 
$$
\mathcal{G}=\{10^{-4} \cdot 2^{j} : j = 0, \dots, 7\}.
$$ 
For OGA-EA and OGD-EA, we use one expert for each $G\in\mathcal{G}$ with learning rate $\eta=G/\sqrt{T}$. SVB-EA uses the same grid for $G$. Each SVB expert has learning rate $\eta_t=G/(\sqrt{t}\,\sigma_t^2)$. 
All Bayesian experts are initialized with $m_0 = 0$ For the meta-learning rate, we use a rolling window of the previous 100 rounds. Specifically, for $t>100$, we set
\begin{align*}
    \gamma_t
    &=
    \sqrt{\frac{\log(100K)}
    {\sum_{s=t-100}^{t-1} \sum_{k=1}^K \omega_{s,k} L_s(q_{s,k})^2}}
\end{align*}
for the Bayesian aggregation methods, and replace $L_s(q_{s,k})$ with the point loss $\ell_s(m_{s,k})$ for OGD.
During the first 100 rounds, we set $\gamma_t=1$. The sharing parameter is fixed at $\sigma_t=10^{-3}$ throughout.

\begin{remark}
\label[remark]{remark:oga_ogd}
For squared-loss regression, the mean loss equals the point loss plus a
posterior-variance term: $L_t(q_t)= (y_t-x_t^\top m_t)^2 + x_t^\top\operatorname{diag}(\sigma_t^2)x_t.$ The variance term does not depend on $m_t$, so the gradient of the mean loss with respect to the posterior mean coincides with the gradient of the point loss. Consequently, with matched learning rates, OGA and OGD have identical mean updates and hence identical point predictions. Their uncertainty epresentations remain different because OGA additionally updates a variational distribution. This equivalence does not extend to other loss functions such as the hinge loss, for which Gaussian smoothing changes the gradient with respect to the posterior mean.
\end{remark}

\paragraph{Metrics}
For each dataset, we plot the average cumulative point-loss $t^{-1}\sum_{r=1}^t \ell_r(\hat{\theta}_r)$  and average cumulative mean-loss $t^{-1}\sum_{r=1}^t L_r(q_r)$ as a function of time.
For each expert grid, we also report the best fixed expert in hindsight,
\begin{align*}
    k^\star
    =
    \arg\min_{k\in[K]}
    \frac{1}{T}
    \sum_{t=1}^T
    \ell_t(\hat{\theta}_{t,k}),
\end{align*}
where $\hat{\theta}_{t,k}=m_{t,k}$ for both OGA and SVB experts. This oracle is not an online method and is used only as a diagnostic benchmark.

\begin{figure}
    \centering
    \includegraphics[width=\linewidth]{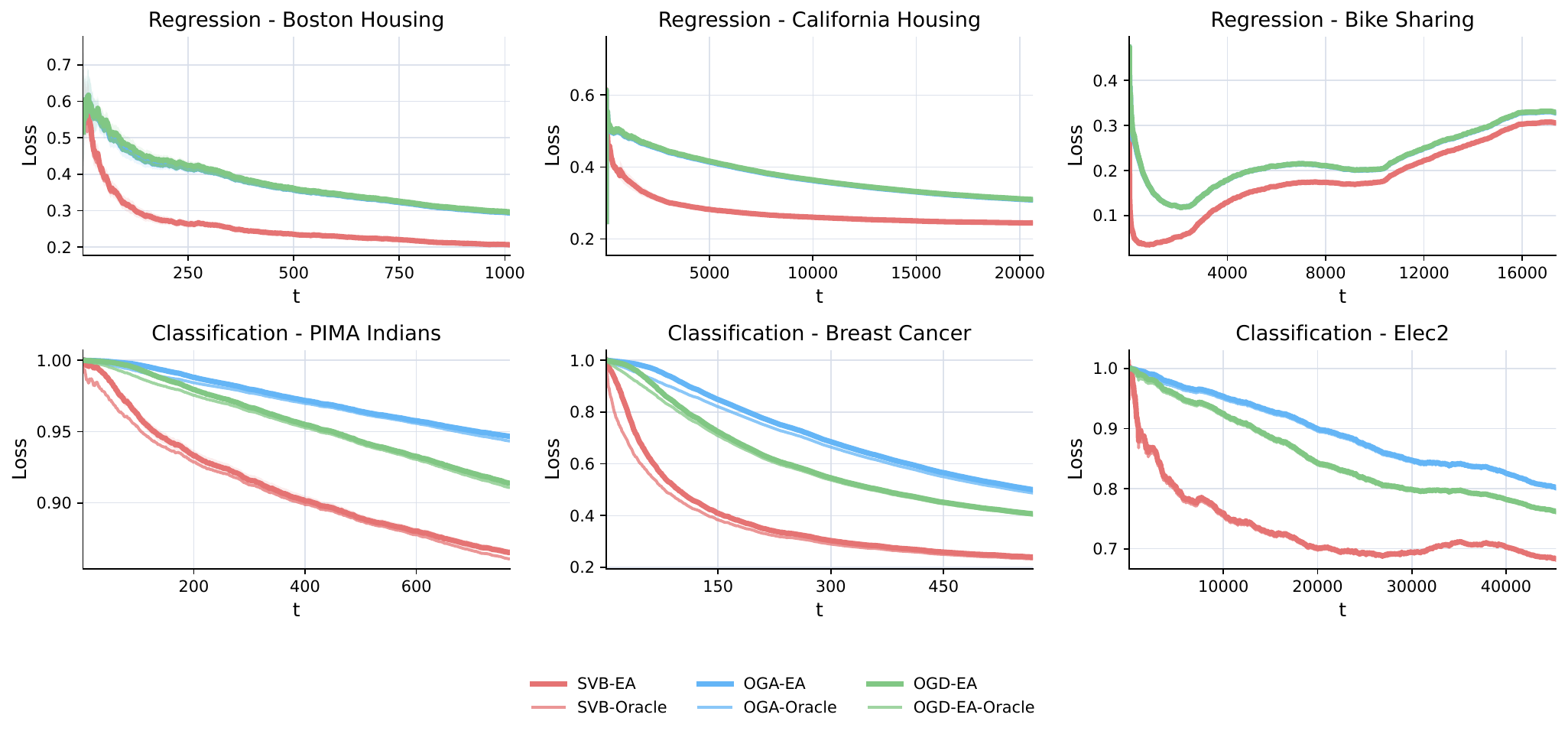}
    \caption{Average cumulative point-prediction loss on the regression and classification benchmarks. Thick curves show the expert-aggregation methods, and the corresponding thin curves show the best fixed expert in hindsight within each expert class. For squared-loss regression, OGA and OGD have identical posterior-mean updates and therefore produce overlapping point-prediction curves. (see \cref{remark:oga_ogd})}
    \label{fig:benchmark_baselines}
\end{figure}

\begin{figure}
    \centering
    \includegraphics[width=\linewidth]{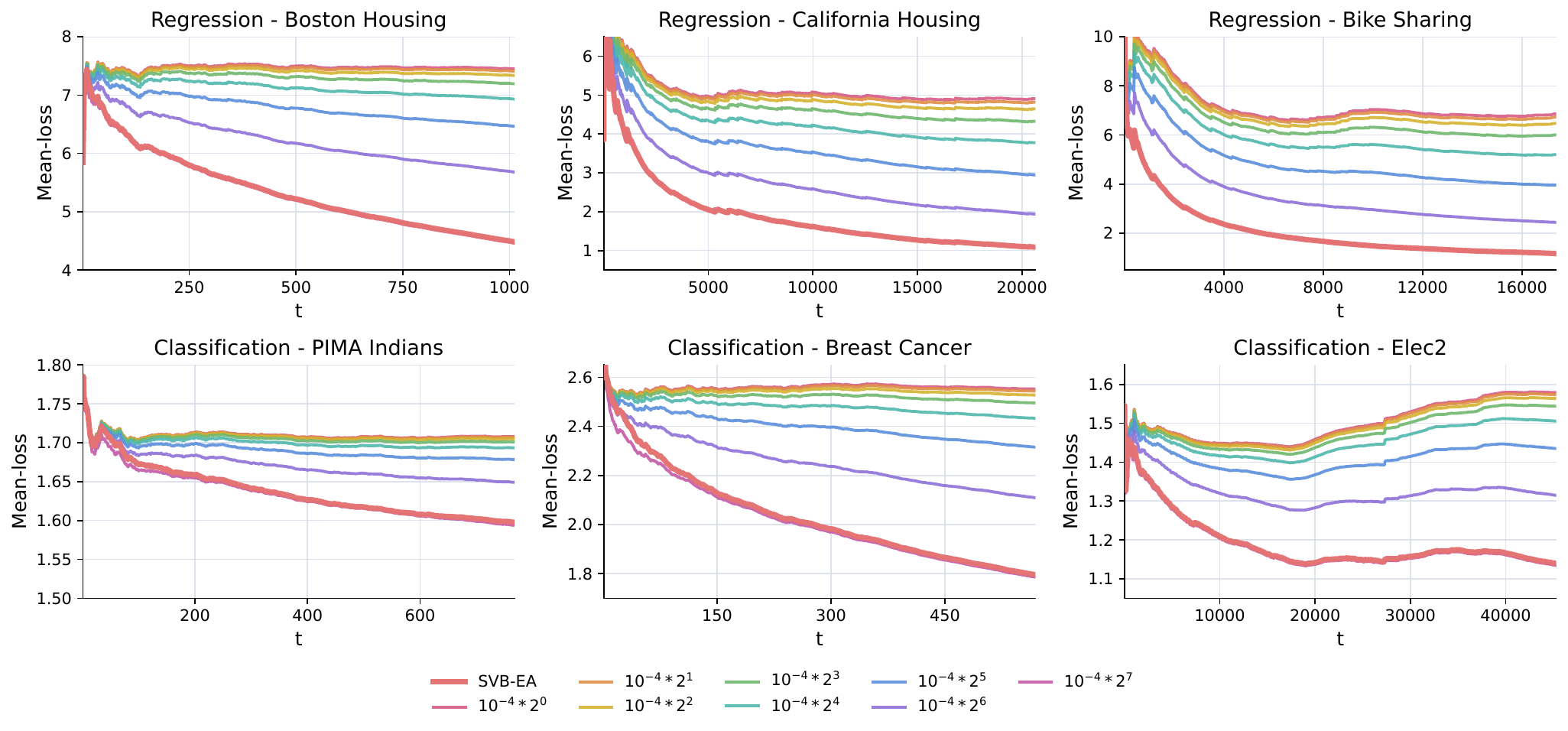}
    \caption{Average cumulative mean loss of SVB-EA (thick curve) and the eight constituent SVB experts (thin curves).}
    \label{fig:benchmark_experts}
\end{figure}

\paragraph{Results}
\cref{fig:benchmark_baselines} compares the average cumulative point-prediction losses of the three aggregation methods with those of the corresponding best fixed experts in hindsight. Across the six datasets, the adaptive aggregates remain close to their oracle counterparts despite not knowing the best learning-rate specification in advance. \cref{fig:benchmark_experts} shows that the performance of the individual experts is highly sensitive to the learning-rate multiplier, and the identity of the strongest expert varies across datasets. Nevertheless, SVB-EA consistently tracks the favorable part of the expert family and, in several datasets, achieves a lower cumulative mean loss than every fixed expert.

\subsection{Online conformal inference}

We evaluate Bayes-DtACI in a nonstationary online conformal prediction setting with abrupt changes in the residual scale  under both heavy-tailed and Gaussian errors..At each round $t$, the learner constructs a prediction interval using only past observations, observes $Y_t$, computes the conformity score $s_t$, and then updates both the individual experts and their aggregation weights.

\paragraph{Setup}
To evaluate the behavior of online conformal methods under nonstationary noise, we consider a synthetic setting with heavy-tailed errors, abrupt changes in the residual scale. We generate $T= 6000$ observations according to
\begin{align*}
     u_t \overset{\mathrm{i.i.d.}}{\sim} \mathrm{Unif}[-1,1],
    \qquad
    x_t=\sqrt{3}u_t,
\end{align*}
where each $\epsilon_t$ is generated either from a  Student-$t$ distribution with two degrees of freedom or from the standard normal distribution. The response is generated by
\begin{align*}
    Y_t = 1 + x_t + a_t(1 + 0.5 x_t) \epsilon_t,
\end{align*}
where the scale factor $a_t$ changes blockwise according to
\begin{align*}
a_t =
\begin{cases}
1.0, & 1\le t\le1500,\\
1.8, & 1501\le t\le3000,\\
0.7, & 3001\le t\le4500,\\
1.4, & 4501\le t\le6000.
\end{cases}
\end{align*}
The point forecast $\widehat Y_t$ is the mean of the ten most recent responses, and the conformity score is $s_t := |Y_t - \widehat Y_t|$. We use the posterior mean $\bar\theta_t := \E_{\theta \sim q_t}[\theta]$ as the conformal threshold.

\paragraph{Competitors}
We compare the proposed Bayes-DtACI with the original DtACI method of \cite{gibbs2024conformal}. For both methods, we use 8 experts for a fair comparison. For DtACI, we use the same grid of learning rates as in \cite{gibbs2024conformal}:
\begin{align*}
    \eta_k \in \{0.001, 0.002, 0.004, 0.008, 0.016, 0.032, 0.064, 0.128\}.
\end{align*}
For  our method, each expert is indexed by a pair $(\eta_k, \tau_k)$ and we consider the following values
    \begin{align*}
        \eta_k\in\{ 0.004, 0.008, 0.064, 0.128\},
        \quad \tau_k\in \{0.5,1\},
    \end{align*}
so that the grid of $\eta_k \tau_k^2$ is the same as that of DtACI. We set the meta-learning rate $\gamma_t$ and the sharing parameter $\sigma_t$ following \citet{gibbs2024conformal}.

\paragraph{Results}
\cref{fig:cp} summarizes the results. In both error settings, the cumulative coverage of Bayes-DtACI approaches the nominal level, while its rolling coverage responds smoothly to changes in residual scale. Under Student-$t$ errors, Bayes-DtACI maintains cumulative coverage closer to the target than DtACI. The difference is less pronounced under Gaussian errors, although Bayes-DtACI again exhibits slightly more stable coverage. These results suggest that replacing the discontinuous ACI update with its Gaussian-smoothed counterpart can improve robustness to large conformity scores generated by heavy-tailed noise.

\begin{figure}[H]
    \centering
    \includegraphics[width=1\linewidth]{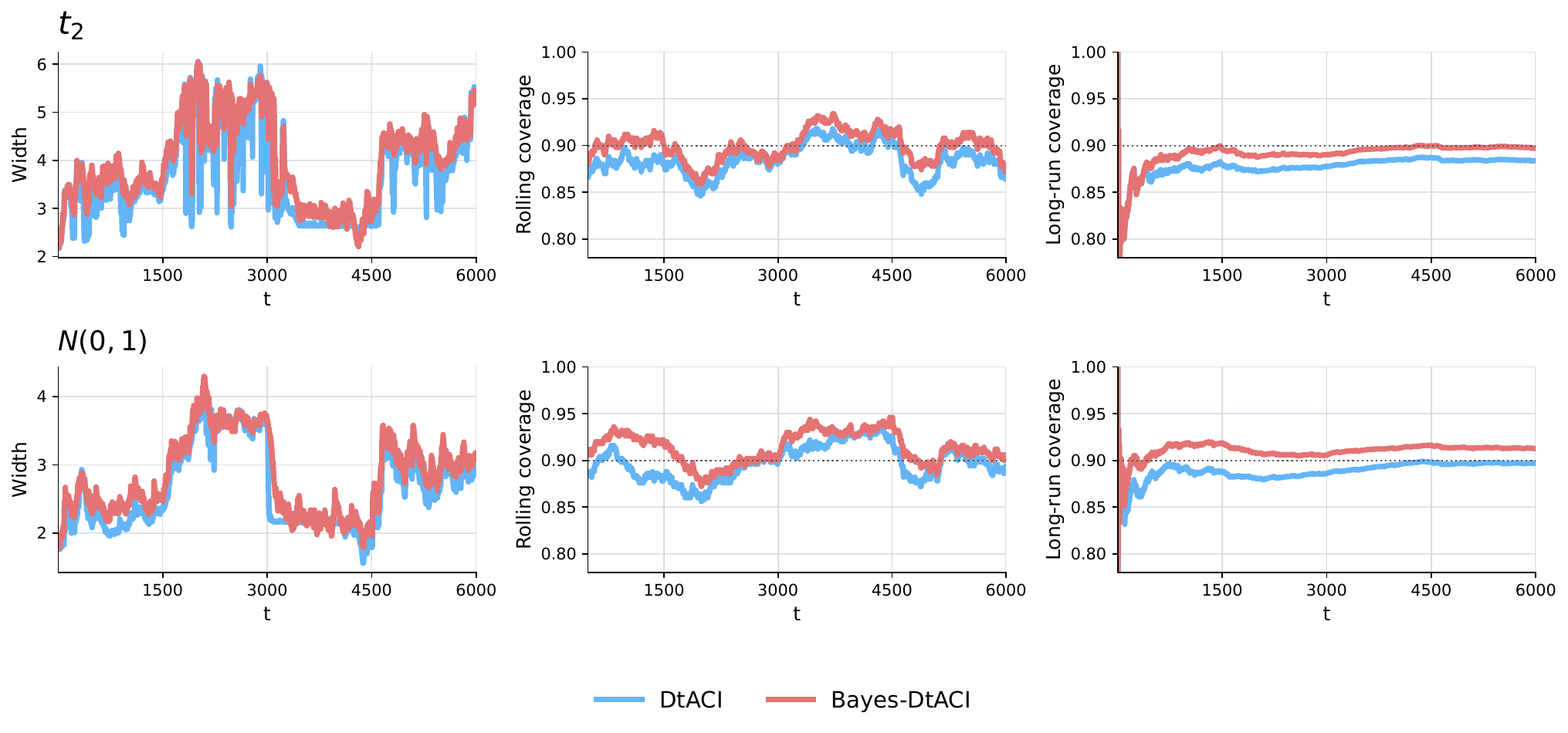}
    \caption{Performance of DtACI and Bayes-DtACI under blockwise changes in the noise scale. The top and bottom rows correspond to Student-$t$ and Gaussian errors, respectively. From left to right, the panels show the conformal threshold $\theta_t$, coverage averaged over a rolling window of 500 rounds, and cumulative coverage up to round $t$. The horizontal dotted line indicates the target coverage level $1-\alpha=0.9$.}
    \label{fig:cp}
\end{figure}

\subsection{Online GP regression}

In this section, we examine the empirical performance of the annealed-loss aggregation procedure developed in \cref{sec:regression}. We maintain a collection of online GP experts with different inverse-bandwidth specifications and combine their one-step-ahead predictive densities according to predictive log-loss. The experiments assess whether the resulting aggregate tracks the best fixed expert and reallocates its weight across bandwidths in response to the effective smoothness and length scale of the regression function.

\paragraph{Setup}
We generate online regression streams of the form
\begin{align*}
    X_t \sim \mathrm{Unif}([0,1]^{10}), 
    \qquad
    Y_t = f^\star(X_t) + \varepsilon_t,
    \qquad
    \varepsilon_t \sim \N(0, 1).
\end{align*}
We consider the following three settings with different effective length-scales and input dimensions.

\begin{enumerate}
    \item (Setting A: Friedman \citep{ki2024mars}) We use the function
    \begin{align*}
        f_A^\star(x)
        =
        10\sin(\pi x_1x_2)
        +20(x_3-0.5)^2
        +10x_4
        +5x_5.
    \end{align*}
    This function combines a nonlinear interaction, a quadratic component, and linear effects.

    \item (Setting B: Doppler \citep{ki2024mars}) We use the function
    \begin{align*}
        f_B^\star(x)
        =
        5\sin\left(
        \frac{4}{\sqrt{x_1^2+x_2^2+0.001}}
        \right)
        +7.5.
    \end{align*}
   This function oscillates increasingly rapidly near the origin and therefore favors a relatively large inverse-bandwidth.

    \item (Setting C: smooth) We use the function
    \begin{align*}
        f_C^\star(x)
        =
        2\left\{\sin\left(\frac{\pi x_1}{2}\right)-\frac{2}{\pi}\right\}
        +
        2\left\{\cos\left(\frac{\pi x_2}{2}\right)-\frac{2}{\pi}\right\}
        +
        0.5(x_3-0.5).
    \end{align*}
    This function is smooth and depends only on the first three coordinates; the remaining seven coordinates are irrelevant nuisance variables.
\end{enumerate}

Together, these settings assess whether the aggregation weights adapt to differences in the effective smoothness and length scale of the underlying regression function.
To evaluate adaptation under nonstationarity, we additionally consider a piecewise-stationary stream of length $T = 3000$, with
\begin{align*}
    f_t^\star =
    \begin{cases}
        f_A^\star, &1 \le t \le 1000, \\
        f_B^\star, &1001 \le t \le 2000, \\
        f_C^\star, &2001 \le t \le 3000.
    \end{cases}
\end{align*}

\paragraph{GP-EA}
In the stationary experiments, we initially fix the observation-noise at its data-generating value, $\varsigma^2 = 1$ for each expert, so that the experiment isolates adaptation to the GP bandwidth. Then the GP experts are indexed by the dyadic grid of inverse bandwidths
\begin{align*}
    \cA = \{2^j : j = -3, -2, \dots, 2\}.
\end{align*}
In the nonstationary experiment, we further augment the expert class with the noise-scale values $ \mathcal S=\{0.5,1,2\},$
resulting in the product grid $\cA\times\mathcal S$ for fair comparison to baseline methods. Moreover, each expert is trained using only the most recent $W=250$ observations to limit the influence of observations from previous regimes, following the sliding-window design of \citet{jun2026online}. A detailed description of the algorithm is deferred to \cref{app:gp_slide}.
We set the meta-learning rate $\gamma_t = 1$ and the sharing parameter $\sigma_t = 0$.

\paragraph{Competitors}
For the point-prediction comparison in the nonstationary experiment, we compare GP-EA with five online regression methods: Online Bagging \citep{oza2001online}, Aggregated Mondrian Forests (AMF) \citep{mourtada2021amf}, $k$-nearest-neighbor regression (KNN), the Hoeffding Adaptive Tree (HAT) \citep{bifet2009adaptive}, and the Adaptive Random Forest (ARF) \citep{gomes2018adaptive}.
Online Bagging, AMF, and ARF are ensemble procedures, whereas KNN and HAT provide local nonparametric and adaptive-tree baselines, respectively.
We use $18$ base learners for the ensemble methods.
All competing methods are implemented using the Python package \texttt{River}, version~0.20.1.

\paragraph{Metrics}
For GP-based predictive-density methods, we report average cumulative negative log-likelihood,
\[
    \frac{1}{t}\sum_{s=1}^{t}
    -\log p_{s, k}(Y_s\mid X_s).
\]
For comparisons with online prediction methods, we use the posterior mean of GP-EA and report the squared loss averaged over the most resent $250$ observations. All experiments are repeated independently $30$ times.

\paragraph{Results}
\Cref{fig:gp_nll} reports the average cumulative negative log-likelihood in the three stationary settings. Across all settings, GP-EA remains close to the best performing fixed bandwidth expert, without requiring the bandwidth to be selected in advance.
\begin{figure}
    \centering
    \includegraphics[width=1\linewidth]{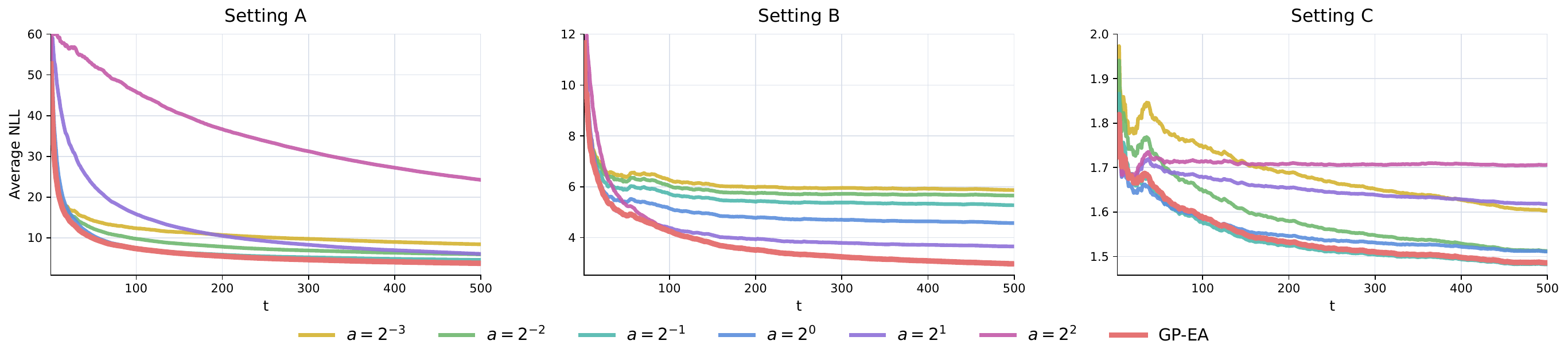}
    \caption{Average cumulative negative log-likelihood of GP-EA and the fixed-bandwidth GP experts in the three stationary synthetic settings.}
    \label{fig:gp_nll}
\end{figure}
\Cref{fig:gp_weights} shows the evolution of the aggregation weights. The three regression functions favor different inverse bandwidths. In Setting A, the weights initially explore several bandwidths before concentrating on an intermediate inverse bandwidth. Setting B, whose regression surface varies rapidly near the origin, places most of its weight on a larger inverse bandwidth. In contrast, the smoother function in Setting C favors a smaller inverse bandwidth. Thus, the aggregation weights respond
to the effective length scale of the regression function in the expected direction.
\begin{figure}
    \centering
    \includegraphics[width=1\linewidth]{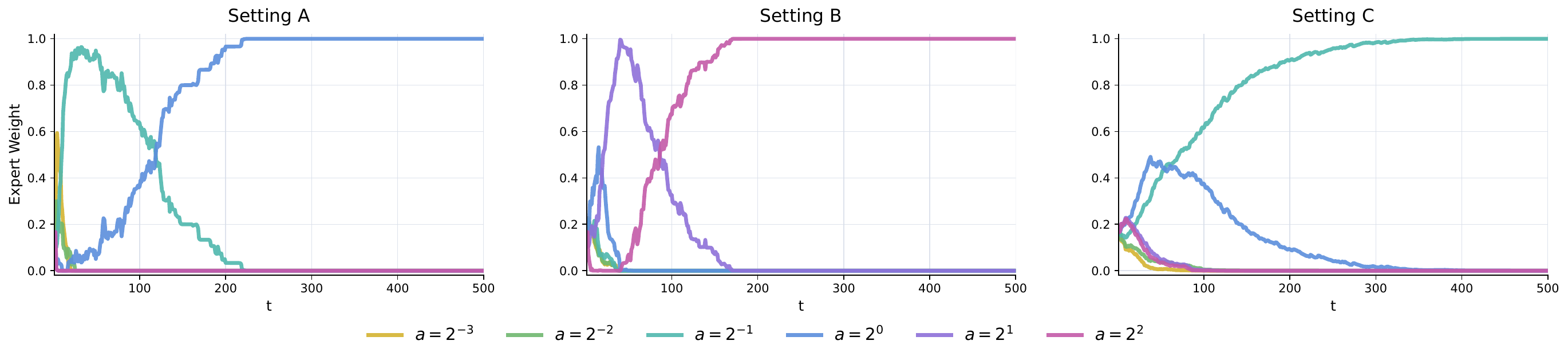}
    \caption{Evolution of the aggregation weights assigned to the six fixed-bandwidth GP
experts in the three stationary synthetic settings.} 
    \label{fig:gp_weights}
\end{figure}

\cref{fig:gp_nonstationary} presents the results for the nonstationary $A\to B\to C$ stream. GP-EA remains competitive with the online regression baselines throughout the stream and exhibits relatively small increases in rolling squared loss after the two regime changes. The right panel shows that these recoveries are accompanied by substantial reallocation of weight across inverse bandwidths. In regime A, the aggregate
places most of its mass on an intermediate bandwidth; after the transition to regime B, the mass shifts toward larger inverse bandwidths; and after the transition to the smoother regime C, it moves back toward smaller inverse bandwidths. These results demonstrate that sliding-window aggregation can revise its effective GP length scale as the data-generating mechanism changes.
\begin{figure}
    \centering
    \includegraphics[width=1\linewidth]{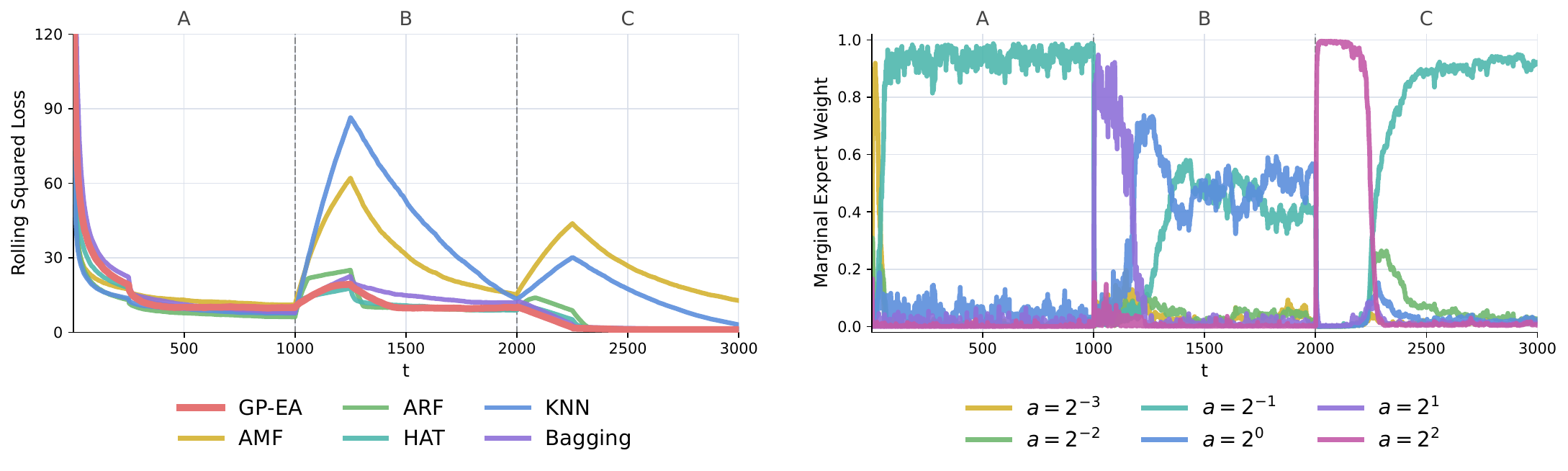}
    \caption{Performance on the nonstationary $A\rightarrow B\rightarrow C$ stream.
The left panel shows squared loss averaged over the most recent 250 rounds
for GP-EA and the competing online regression methods. The right panel shows
the GP-EA weights marginalized over the noise-scale grid for each inverse
bandwidth. Vertical dashed lines indicate the regime changes at
$t=1000$ and $t=2000$.}
    \label{fig:gp_nonstationary}
\end{figure}

\section{Conclusion}
\label{sec:conclusion}

\paragraph{Summary}
We introduced an expert-aggregation framework for adaptive Bayesian online learning. The key idea is to treat different Bayesian update rules that use different learning rates, prior distributions, variational families, or other inferential specifications as distribution-valued experts and to combine their sequential predictions according to their observed predictive performance. This two-level construction separates the within-expert Bayesian update from the across-expert adaptation mechanism and therefore provides a modular way to reduce sensitivity to inferential choices that would otherwise need to be fixed in advance.

A central message of this work is that the appropriate aggregation guarantee depends on how the performance of a Bayesian expert is evaluated. Mean-loss aggregation directly controls posterior mean functionals and applies to criteria such as the pinball loss, but generally incurs a $O(\sqrt{T})$ aggregation cost. In contrast, annealed-loss aggregation evaluates the full predictive distribution and yields a logarithmic aggregation cost. Under log loss, the annealed score coincides exactly with posterior predictive log loss, so the resulting regret bound has a direct interpretation in terms of cumulative predictive KL risk. These two scoring rules should therefore be viewed not as competing versions of the same procedure, but as tools matched to different inferential targets.

We illustrated this principle in two applications. For online conformal
inference, Gaussian streaming variational updates lead to a smoothed Bayesian counterpart of adaptive conformal inference, and aggregation over multiple learning-rate and smoothing specifications yields long-run randomized coverage. For online Gaussian process regression, aggregation over a bandwidth grid produces a predictive-density oracle inequality and adapts to unknown H\"older smoothness up to logarithmic factors. The numerical experiments further show that the aggregate tracks well-performing experts across a range of online variational, conformal, and nonparametric regression settings, without requiring oracle expert selection.

\paragraph{Future work}
Several directions for future research remain open. First, although we focus primarily on fixed-share aggregation, this is only one of  possible choices for meta learning mechanism. More adaptive online aggregation methods, including strongly adaptive and parameter-free procedures, may yield sharper guarantees when the best learning rate, prior, variational family, or other inferential specification changes over time \citep{jun2017improved,zhang2021dual,zhang2022simple}. Incorporating such methods could provide stronger local or interval-wise regret guarantees and improve the ability of the Bayesian aggregate to respond rapidly to nonstationary data streams.

Second, the present theory is developed mainly for a finite, prespecified collection of experts. Extending the framework to continuous, data-dependent, or growing expert families would permit adaptation over richer classes of learning rates, priors, variational families, and other inferential specifications.. Such extensions would require controlling both the statistical complexity of the expert class and the computational cost of maintaining the aggregate.

\appendix

\section{Effects of aggregation method on uncertainty quantification}

In this experiment, we consider an intercept-only toy online regression problem to examine how the choice of aggregation loss affects uncertainty quantification.

\paragraph{Simulation setup}
Each stream has length $T = 6000$, and results are averaged over $30$ independent replications. 
We consider two data-generating mechanisms:
\begin{enumerate}
    \item (well-specified) We generate $y_t = \varepsilon_t$ with $\varepsilon_t \sim \N(0, 0.1^2)$
    \item (misspecified) We generate $y_t = 2 s_t + \varepsilon_t$ with $s_t \sim \Unif\{-1, +1\}$.
\end{enumerate}

\paragraph{Baseline experts}
We use the SVB update of \cite{cherief2019generalization} to approximate
one-step Bayesian updating under the squared loss $\ell_t(\theta)=(y_t-\theta)^2.$
Each expert maintains a Gaussian variational posterior
$q_{t,k}(\theta)=\N(m_{t,k},\sigma_{t,k}^2)$ with $m_{t,k} \in [-20, 20]$ and $\sigma_{t,k}\in[0,1].$
Experts are indexed by the learning-rate multiplier
$G_k\in\{0.001,0.004,0.016,0.064,0.256\}.$
All experts are initialized at $m_{1,k}=0$ and $\sigma_{1,k}=1$, and predictions are made using the posterior mean $m_{t,k}$.

\paragraph{Evaluation metric}
We report regret under the point-prediction:
\begin{align*}
    \cR_t = \sum_{s = 1}^t \ell_s(\bar \theta_{s}) - \min_{k \in [K]}\sum_{s=1}^t \ell_s( \bar \theta_{s,k})
\end{align*}

\begin{figure}[H]
    \centering
    \includegraphics[width=1\linewidth]{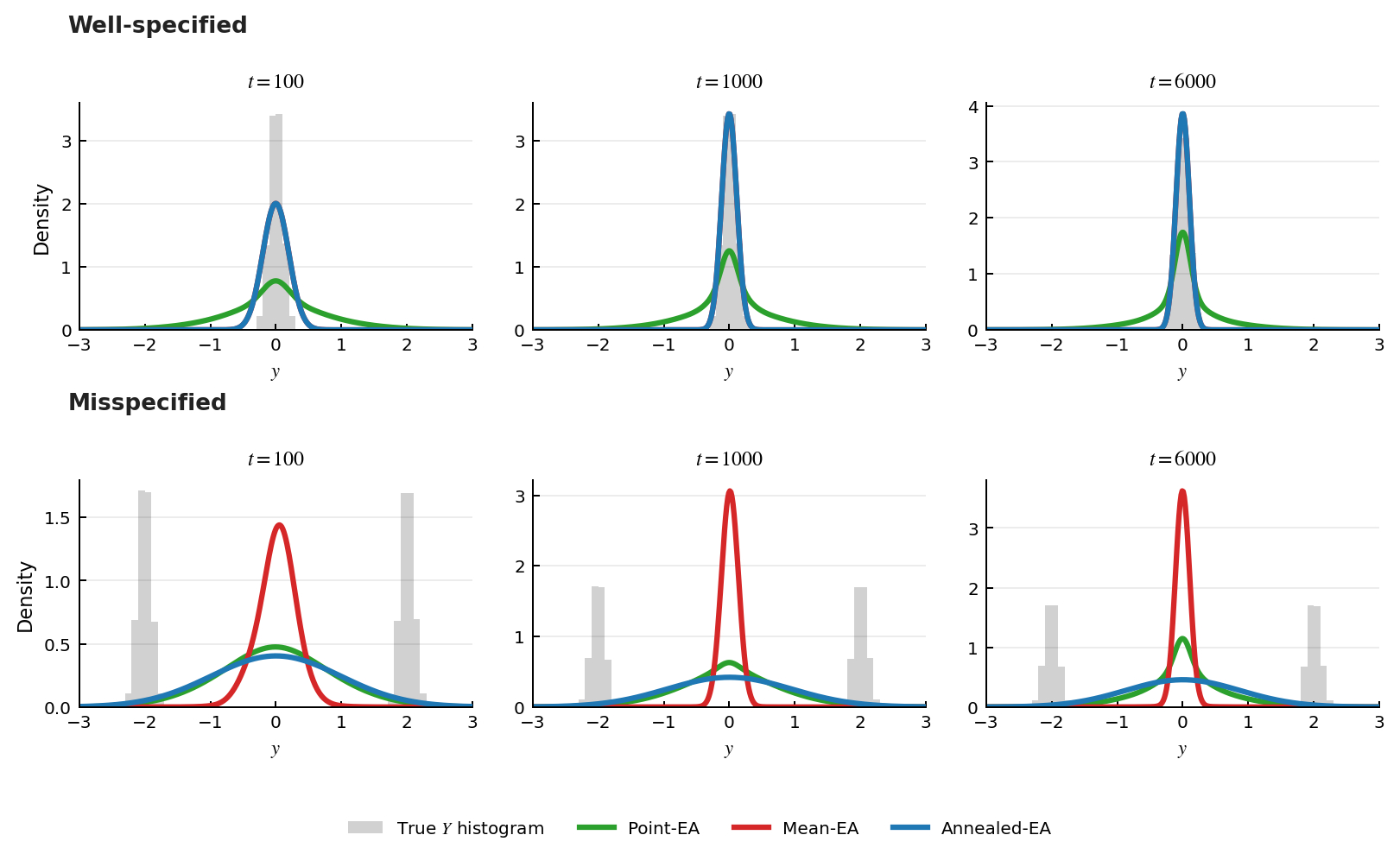}
    \caption{Visualization of $p_t(y) = \int \N(y ;\theta, 0.1^2)q_t(\theta) d\theta$. 
    Annealed-EA retains substantially larger uncertainty under misspecification, whereas Mean-EA rapidly collapses, making the annealed score safer against over-confident predictions.}
    \label{fig:placeholder}
\end{figure}

\begin{figure}[H]
    \centering
    \includegraphics[width=1\linewidth]{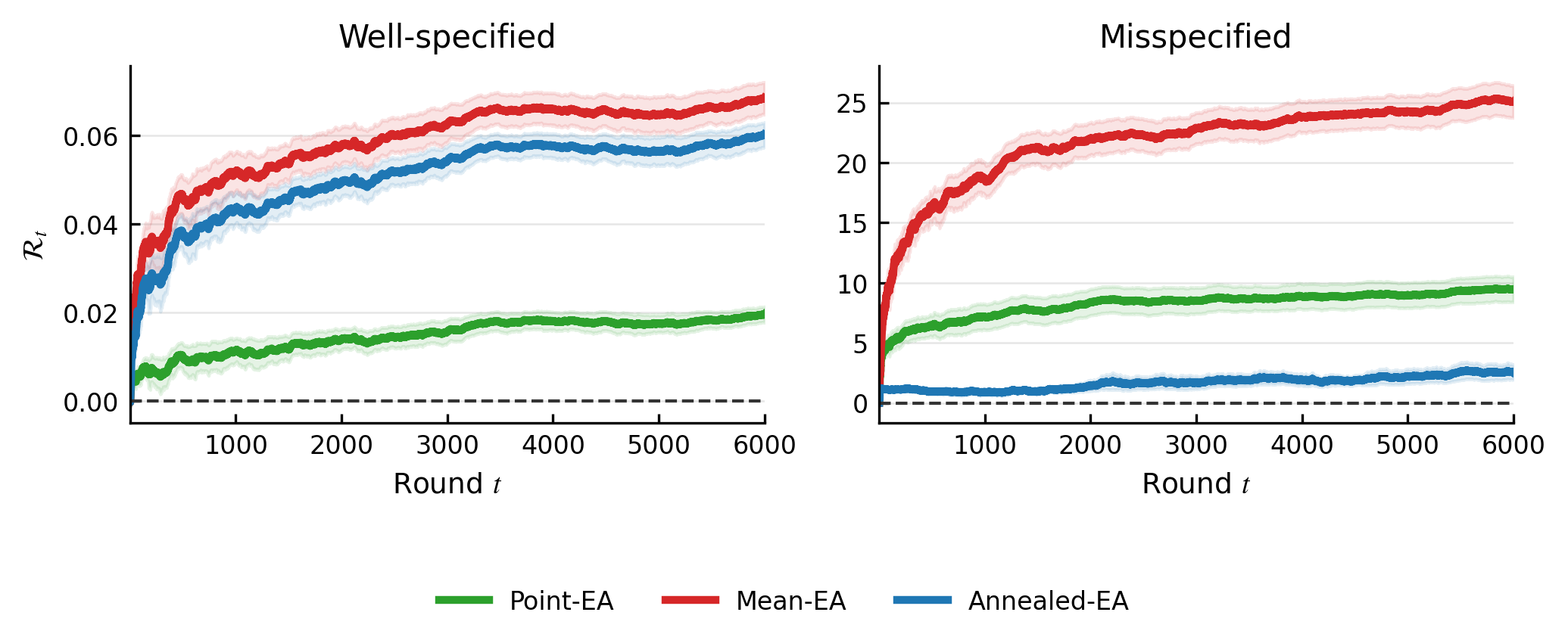}
    \caption{Regret under the point-prediction target. Point-EA has the smallest regret in the well-specified setting despite retaining a poor predictive distribution, while Annealed-EA is more robust under bimodal misspecification.}
    \label{fig:placeholder}
\end{figure}

\paragraph{Effect of dispersion}
Let $r_t^2=(y_t-\bar u_t)^2$. If prediction is evaluated only through the posterior mean, the point loss is $\ell_t(\bar \theta_t) = r_t^2$, and hence ignores the posterior variance $v_t$. The mean score, on the other hand,
\begin{align*}
    L_t(q_\psi)=r_t^2+v_t,
\end{align*}
penalizes $v_t$ linearly, independently of the residual size. The annealed-score behaves differently:
\begin{align*}
    \Lambda_t(q_\psi)
    =
    \frac{1}{2\gamma}\log(1+2\gamma v_t) + \frac{r_t^2}{1+2\gamma v_t}.
\end{align*}
It balances a logarithmic cost of dispersion against a residual term damped by $(1+2\gamma v_t)^{-1}$. Thus, when the residual is large, a more dispersed expert may incur a smaller annealed meta-loss than a more concentrated one. 
This is a meta-level effect: annealed scoring does not change the SVB update of each expert, but changes the relative weights assigned to already updated experts. 
Consequently, in misspecified settings where concentrated posterior means can be misleading, annealed scoring can help retain aggregate predictive uncertainty by giving weight to experts with non-negligible low-loss posterior mass.

\section{Closed-form expressions for the mean- and annealed-losses in the mean-field Gaussian family}
\label[appendix]{appendix:loss_compute}
This section provides the closed-form scores used in \cref{sec:experiment}. Throughout this section, we consider the mean-field Gaussian variational family.

\paragraph{Squared loss}
For regression with $y_t \in \mathbb R$, we consider the squared loss $\ell_t(\theta) = (y_t - \theta^{\top} x_t)^2$. Under $q_\psi=\N(m,\operatorname{diag}(\sigma^2))$, define $\bar  u_t = x_t^\top m$ and $v_t = x_t^\top \operatorname{diag}(\sigma^2)x_t.$ Then the mean-loss is given by
\begin{align*}
    L_t(q_{\psi}) = (y_t - \bar u_t)^2 + v_t
\end{align*}
and the annealed-loss with parameter $\gamma >0$ is given by
\begin{align*}
    \Lambda_t(q_{\psi}) = \frac{1}{2 \gamma} \log \left( 1 + 2 \gamma \, v_t \right) + \frac{(y_t - \bar u_t)^2}{1 + 2 \gamma \,v_t}.
\end{align*}

\paragraph{Hinge loss}
For binary classification with $y_t\in\{-1,+1\}$, we consider the hinge loss $\ell_t(\theta)=(1-y_t x_t^\top\theta)_+.$
Under $q_\psi=\N(m,\operatorname{diag}(\sigma^2))$, define $\bar z_t = y_t x_t^\top m$ and $v_t = x_t^\top \operatorname{diag}(\sigma^2)x_t.$ Then the mean-loss is given by
\begin{align*}
    L_t(q_\psi) =
    (1-\bar z_t) \Phi\left(\frac{1-\bar z_t}{\sqrt{v_t}}\right) +
    \sqrt{v_t} \phi\left(\frac{1-\bar z_t}{\sqrt{v_t}}\right).
\end{align*}
and the annealed-loss with parameter $\gamma >0$ is given by
\begin{align*}
    \Lambda_t(q_\psi)
    &= -\frac{1}{\gamma}
    \log \left[
    1-\Phi\left(\frac{1- \bar z_t}{\sqrt{v_t}}\right) +
    \exp\left(
        -\gamma+\gamma \bar z_t+\frac{1}{2}\gamma^2v_t
    \right)
    \Phi\left(
        \frac{1-\bar z_t-\gamma v_t}{\sqrt{v_t}}
    \right)
    \right].
\end{align*}

\section{Update rules used in \cref{sec:benchmark}}
\label[appendix]{appendix:oga_svb}
This section provides the explicit update rules for the base learners used in \cref{sec:benchmark}. Throughout, we consider the mean-field Gaussian variational family
\begin{align*}
    q_{\psi_t}(\theta) = \N(m_t, \operatorname{diag}(\sigma_t^2)).
\end{align*}
Following \cite{cherief2019generalization}, we employ the following closed-form updates for SVB and OGA, together with standard online gradient descent as a frequentist baseline.

\paragraph{Streaming Variational Bayes (SVB)}
The mean and standard deviation parameters are updated as
\begin{align*}
m_{t+1} &\gets m_t - \eta \sigma_t^2 \nabla_m L_t(q_{\psi_t}), \\
\sigma_{t+1} &\gets \sigma_t h \left( \frac{1}{2} \eta \sigma_t \nabla_{\sigma} L_t (q_{\psi_t}) \right),
\end{align*} where $h(x) := \sqrt{1 + x^2} - x$ is applied coordinatewise.

\paragraph{Online Gradient Approximation (OGA)}
OGA performs Euclidean gradient updates on the mean-loss:
\begin{align*}
m_{t+1} &\gets m_t - \eta s^2 \nabla_m L_t(q_{\psi_t}), \\
\sigma_{t+1} &\gets \sigma_t - \eta s^2 \nabla_\sigma L_t(q_{\psi_t}),
\end{align*}
where $s^2$ denotes the prior variance.

\section{Efficient algorithm for online GP regression}

\subsection{Posterior predictive computation}
The predictor at time $t$ is the posterior predictive mean
\begin{equation}\label{eq:predictor}
    \hat f_{t,k}(x) = \E_{ q_{t,k}}[f(x)]
=\mathbf{k}_{t-1,k}(x)^\top \del{K_{t-1,k}+\varsigma^{2}I_{t-1}}^{-1} Y_{1:t-1},
\end{equation}
with $\mathbf{k}_{t-1,k}(x)=\del{c_{a_k}(x,X_1),\ldots,c_a(x,X_{t-1})}^\top$ and $K_{t-1,k}=\del{c_{a_k}(X_i,X_j)}_{i,j=1}^{t-1}$. The posterior predictive variance at $x$ is given by
    \begin{align}
          v_{t,k}(X_t) = c_{a_k}(X_t, X_t) - \mathbf{k}_{t-1,k}^\top \bigl(K_{t-1,k} + \varsigma^2 I_{t-1}\bigr)^{-1} \mathbf{k}_{t-1,k}.
    \end{align}
Because $q_{t,k}$ is Gaussian, the predictive density is itself Gaussian:
    \begin{equation}\label{eq:predictive}
        \p_{t,k}(y | x, \cD_{t-1})  =  \frac{1}{\sqrt{\varsigma^2 + v_{t,k}(x)}} \phi\del{\frac{y - \hat f_{t,k}(x)}{\sqrt{\varsigma^2 + v_{t,k}(x)}}}.
    \end{equation}

\subsection{Sliding-window updates for GP}
\label[appendix]{app:gp_slide}
Given a window size $W$,  we define the active index set
\begin{align*}
    \cI_t := \{ \max\{t-W, 1\}, \max\{t-W, 1\} +1, \dots, t-1\} \quad n_t := | \cI_t|
\end{align*}
at round $t$, with the convention $\cI_t=\{1,\dots,t-1\}$  when $W=\infty$.
For a GP expert indexed by the inverse bandwidth $a_k$ and noise variance $\varsigma_k^2$, let
\begin{align*}
    A_{t,k} &:= K_{\cI_t, k} + \varsigma_k^2 I_{n_t}, \quad Q_{t,k}:= A_{t,k}^{-1},
\end{align*}
where $K_{\cI_t,k} := \bigl(c_{a_k}(X_i,X_j)\bigr)_{i,j\in \cI_t}.$
After observing $(X_t, Y_t)$, the window is advanced from $\cI_t$ to $\cI_{t+1}$. When the window is full, the oldest observation is first removed. Ordering the observations so that the oldest observation appears first, partition
\begin{align*}
    Q_{t,k}
    = \begin{pmatrix}
        q_{11} & q_{12}^\top \\
        q_{12} & Q_{22}
    \end{pmatrix}.
\end{align*}
The inverse corresponding to the retained observations is obtained by the \textit{downdate}
\begin{align*}
    \check Q_{t,k}
    = Q_{22} - \frac{q_{12}q_{12}^\top}{q_{11}}
\end{align*}
due to block matrix inversion formula. For $t \le W$, no observation is removed and we simply set $\check Q_{t,k} = Q_{t,k}.$
Let $\check{\mathcal I}_t := \cI_t \setminus \{t-W\}$ denote the retained index set and define
\begin{align*}
    \mathbf u_{t,k} &:=
    \bigl(c_{a_k}(X_i,X_t)\bigr)_{i\in\check{\mathcal I}_t}, \\
    d_{t,k} &:=
    c_{a_k}(X_t,X_t)+\varsigma_k^2 - \mathbf u_{t,k}^\top \check Q_{t,k}\mathbf u_{t,k}.
\end{align*}
Appending the new observation gives the rank-one block update
\begin{align*}
    Q_{t+1,k}
    =
    \begin{pmatrix}
        \check Q_{t,k} +d_{t,k}^{-1}(\check Q_{t,k}\mathbf u_{t,k})(\check Q_{t,k}\mathbf u_{t,k})^\top
        & -d_{t,k}^{-1}\check Q_{t,k}\mathbf u_{t,k} \\
        -d_{t,k}^{-1} \mathbf u_{t,k}^\top\check Q_{t,k}
        & d_{t,k}^{-1}
    \end{pmatrix}.
\end{align*}

\section{Proofs for \cref{sec:regret}}

\subsection{Proof of \cref{thm:regret_mean}}

\begin{proof}
Let $L_{t,k}:=L_t(q_{t,k})$ for simplicity.
We recursively define a sequence of vectors $(z_{t,1},\dots, z_{t,K})$ as, starting from an initial value $(z_{1,1},\dots,z_{1,K})=(\omega_{1,1},\dots, \omega_{1,K})$,
    \begin{align*}
    \bar{z}_{t,k}&:=z_{t,k}\exp( -\gamma L_{t,k})\\
        z_{t+1,k}&:=(1-\sigma)\bar{z}_{t,k}+\sigma\frac{1}{K} \bar{Z}_t
    \end{align*}
with $ \bar{Z}_t:=\sum_{k\in[K]} \bar{z}_{t,k}.$ By construction, we have 
    \begin{align}
    \label{eq:nconst_eq}
        Z_{t+1} := \sum_{k\in[K]}z_{t+1,k}
        = (1-\sigma)\sum_{k\in[K]}\bar{z}_{t,k}+\sigma\bar{Z}_t
       = \bar{Z}_t.
    \end{align}
Hence, letting $\bar\omega_{t,k}=z_{t,k}\exp(-\gamma L_{t,k})/\bar{Z}_t$, this recursively updated procedure recovers our weight update algorithm as
    \begin{align*}
        \omega_{t+1,k}=\frac{  z_{t+1,k}}{\sum_{h\in[K]}z_{t+1,h}}
        &=(1-\sigma)\frac{z_{t,k} \exp( -\gamma L_{t,k})}{\bar{Z}_t}+\sigma\frac{1}{K}\\
        &=(1-\sigma)\bar\omega_{t,k}+\sigma\frac{1}{K}.
    \end{align*}
Then the proof follows the same lines of Lemma 2 of \cite{gibbs2024conformal} and Lemma A.2 of \cite{gradu2023adaptive}. By the inequalities $\exp(-x) \le 1- x +x^2$ and $1- y \le \exp(-y)$, we have
    \begin{align*}
        \frac{Z_{t+1}}{Z_t}
        =\frac{\bar{Z}_t}{Z_t} 
        &=\sum_{k\in[K]}\frac{z_{t,k}\exp( -\gamma L_{t,k})}{Z_t}\\
       &= \sum_{k\in[K]} \omega_{t,k} \exp( -\gamma L_{t,k}) \\
        &\le \exp\del[3]{-\gamma  \sum_{k\in[K]} \omega_{t,k}  L_{t,k} + \gamma^2 \sum_{k\in[K]} \omega_{t,k} L_{t,k}^2}.
    \end{align*}
where we use the identity \eqref{eq:nconst_eq} for the first equality. Thus, inductively, we have
    \begin{align*}
    \frac{Z_{u+1}}{Z_v}
         \le \exp \del[3]{ - \sum_{t \in\cI} \sum_{k\in[K]} (-\gamma \omega_{t,k} \, L_{t,k}+  \gamma^2 \omega_{t,k} L_{t,k}^2)}.
    \end{align*}
On the other hand, since $z_{t+1,k} \ge z_{t,k} (1-\sigma) \exp(-\gamma L_{t,k})$, we have
    \begin{align*}
        \frac{Z_{u+1}}{Z_v} \ge \frac{z_{u,k}}{Z_v} 
        &\ge (1-\sigma)^{|\cI|} \exp \del[2]{  - \sum_{t \in\cI} \gamma L_{t,k} }\frac{z_{v,k}}{Z_v}\\
        &\ge (1-\sigma)^{|\cI|} \exp \del[2]{  - \sum_{t \in\cI} \gamma L_{t,k} }\frac{\sigma}{K}.
    \end{align*}
    Combining these two inequalities and taking a logarithm yields
    \begin{align*}
        \log(\sigma/K) + |\cI| \log(1-\sigma) - \sum_{t\in\cI} \gamma L_{t,k} \le  \sum_{t\in\cI} \sum_{k \in [K]} \del{-\gamma \omega_{t,k} L_{t,k} + \gamma^2 \omega_{t,k} L_{t,k}^2}.
    \end{align*}
Finally, by the inequality $\log(1-\sigma) \ge -2\sigma$ for $\sigma \in[0, 1/2]$, 
    \begin{align*}
        \sum_{t \in\cI} L_{t}(q_t)& =  \sum_{t \in\cI} \sum_{k \in [K]} \omega_{t,k} L_{t,k}
        \le \sum_{t \in\cI} L_{t,k} 
        + \gamma \sum_{t \in\cI}\sum_{k \in [K]}\omega_{t,k} L_{t,k}^2+ \frac{1}{\gamma} (\log(K/\sigma) + 2 |\cI|\sigma).
    \end{align*}
 Taking the minimum with respect to $k$ on both sides, we obtain the desired result.
\end{proof}

\subsection{Proof of \cref{thm:regret_annealed}}

\begin{proof}
Write $\Lambda_{t,k} := \Lambda_t(q_{t,k})$ for simplicity. Since $q_t = \sum_k \omega_{t,k}q_{t,k}$ and
$q \mapsto \mathbb{E}_{\theta\sim q}\sbr[1]{\e^{-\gamma\ell_t(\theta)}}$ is linear, it follows that
\begin{align}
\label{eq:affine_q}
    \exp\del[1]{-\gamma\Lambda_t(q_t)}
    = \sum_{k=1}^K \omega_{t,k}\,\mathbb{E}_{\theta\sim q_{t,k}}\sbr[1]{\exp(-\gamma\ell_t(\theta))}
    = \sum_{k=1}^K \omega_{t,k}\exp(-\gamma\Lambda_{t,k}),
\end{align}
regardless of how $\omega_t$ was produced. By the definition of the weight update,
    \begin{align*}
        \bar\omega_{t,k} =\frac{ \omega_{t,k}\exp(-\gamma\Lambda_{t,k})}{\sum_{h=1}^K \omega_{t,h}\exp(-\gamma\Lambda_{t,h})}
        = \omega_{t,k}\exp(-\gamma(\Lambda_{t,k}-\Lambda_t(q_t)))
    \end{align*}
Moreover, by the definition of $\omega_{t+1,k} = (1-\sigma)\bar\omega_{t,k} + \sigma/K$, we have for any $t$,
\begin{align*}
    \omega_{t+1,k} \ge(1-\sigma)\bar\omega_{t,k} = (1-\sigma)\omega_{t,k} \exp(-\gamma(\Lambda_{t,k}-\Lambda_t(q_t))).
\end{align*}
Fix $k$ and the interval $I = [u,v]$. By  chaining the above bound over
$t = u, \dots, v$, we have
\begin{align*}
    1 \ge \bar\omega_{v,k}
    \ge\omega_{u,k}(1-\sigma)^{\abs{I}-1}
  \exp\del[2]{-\gamma\sum_{t\in\cI}\{\Lambda_{t,k}-\Lambda_t(q_t)\}}.
\end{align*}
Moreover,  by the definition of $\omega_{t+1,k} = (1-\sigma)\bar\omega_{t,k} + \sigma/K$ again, we have $ \omega_{u,k} \ge \sigma/K$ and thus
\begin{align*}
    \bar\omega_{v,k}
    \ge \frac{\sigma}{K}(1-\sigma)^{\abs{I}-1}
 \exp\del[2]{-\gamma\sum_{t\in\cI}\{\Lambda_{t,k}-\Lambda_t(q_t)\}}.
\end{align*}
Taking logarithms, rearranging, and using
$\log(1-\sigma) \ge -2\sigma$ for $\sigma \in [0, 1/2]$,
\begin{align*}
    \sum_{t\in\cI}\Lambda_t(q_t)
    \le \sum_{t\in\cI}\Lambda_{t,k}
    + \frac{1}{\gamma}\del[2]{\log\frac{K}{\sigma} + 2\abs{I}\sigma},
\end{align*}
which is the desired result.    
\end{proof}

\subsection{Proof of \cref{thm:regret_noshare}}

\begin{proof}
 Write $\Lambda_{t,k} := \Lambda_t(q_{t,k})$ for simplicity. When $\sigma=0$, the weight update becomes
    \begin{align*}
        \omega_{t+1,k} =\frac{ \omega_{1,k}\exp(-\gamma\sum_{s=1}^{t}\Lambda_{s,k})}{\sum_{h=1}^K \omega_{1,h}\exp(-\gamma\sum_{s=1}^{t}\Lambda_{s,h})}.
    \end{align*}
This, together with \eqref{eq:affine_q} in the proof of \cref{thm:regret_annealed}, implies that
    \begin{align*}
        \exp\del[1]{-\gamma\Lambda_t(q_t)}
        = \sum_{k=1}^K \omega_{t,k}\exp(-\gamma\Lambda_{t,k})
          = \frac{\sum_{k=1}^K \omega_{1,k}\exp(-\gamma\sum_{s=1}^{t}\Lambda_{s,k})}{\sum_{h=1}^K \omega_{1,h}\exp(-\gamma\sum_{s=1}^{t-1}\Lambda_{s,h})}.
    \end{align*}
Taking the logarithm to both sides and summing over $t\in[T]$, we have that for any $k,$
    \begin{align*}
        \sum_{t=1}^T\Lambda_t(q_t)
       &=-\frac{1}{\gamma}\log \del[3]{\sum_{k=1}^K \omega_{1,k}\exp\del[2]{-\gamma\sum_{s=1}^{t}\Lambda_{s,k}}}\\
       &\le- \frac{1}{\gamma}\log \del[3]{ \omega_{1,k}\exp\del[2]{-\gamma\sum_{s=1}^{t}\Lambda_{s,k}}}\\
       &=- \frac{1}{\gamma}\log(\omega_{1,k})+\sum_{s=1}^{t}\Lambda_{s,k},
    \end{align*}
which completes the proof.
   \end{proof}

\section{Proofs for \cref{sec:conformal}}

\subsection{Supporting lemmas}

\subsubsection{Expected pinball loss and its gradient}


\begin{lemma}
\label[lemma]{lem:pinball-grad}
Let $\psi\in\R$, $\tau>0$, $r \in \mathbb{R}$ and $\alpha\in(0,1)$. Then we have
    \begin{align*}
         \E_{\theta\sim \N(\psi, \tau^2)}[\ell_\alpha(\theta,r)]
          = (1 - \alpha)(r - \psi) + (\psi - r)\Phi\del[2]{\frac{\psi - r}{\tau}}
          + \tau\phi\del[2]{\frac{\psi - r}{\tau}},
    \end{align*}    
and
    \begin{align*}
         \frac{\partial}{\partial \psi}  \E_{\theta\sim \N(\psi, \tau^2)}[\ell_\alpha(\theta,r)]
         = \alpha - \Phi\del[2]{\frac{r - \psi}{\tau}}.
    \end{align*}
\end{lemma}

\begin{proof}
The expected loss splits at $\theta = r$:
\begin{align*}
      \E_{\theta\sim \N(\psi, \tau^2)}[\ell_\alpha(\theta,r)]
  &= \alpha \int_r^\infty (\theta - r)\phi((\theta-\psi)/\tau)\d\theta
    + (1 - \alpha) \int_{-\infty}^r (r -\theta)  \phi((\theta-\psi)/\tau)\d\theta.
\end{align*}
The first term is given by
    \begin{align*}
    \int_r^\infty (\theta - r) \phi((\theta-\psi)/\tau)\d\theta&=
        \int_r^\infty (\theta - \psi) \phi((\theta-\psi)/\tau)\d\theta +(\psi-r)\int_r^\infty\phi((\theta-\psi)/\tau)\d\theta\\
        &= -\tau \phi((\theta-\psi)/\tau)|_r^\infty+(\psi-r)\cbr[0]{1-\Phi((r-\psi)/\tau)}\\
        &=\tau\phi((r-\psi)/\tau)+(\psi-r)\Phi((\psi-r)/\tau).
    \end{align*}
Similarly, the second term is given by
    \begin{align*}
        \int_{-\infty}^r (r - \theta) \phi((\theta-\psi)/\tau)\d\theta
        &=(r-\psi) \int_{-\infty}^r \phi((\theta-\psi)/\tau)\d\theta-\int_{-\infty}^r ( \theta-\psi) \phi((\theta-\psi)/\tau)\d\theta\\
        &=(r-\psi)\Phi((r-\psi)/\tau)+\tau \phi((\theta-\psi)/\tau)|_{-\infty}^r\\
        &=(r-\psi)\cbr{1-\Phi((\psi-r)/\tau)}+\tau\phi((r-\psi)/\tau).
    \end{align*}
Summing the two pieces yields the stated expression. Differentiating this, since $\phi'(z)=-z\phi(z)$, we have
    \begin{align*}
         \frac{\partial}{\partial\psi}   \E_{\theta\sim q_{\psi}}[\ell_\alpha(\theta,r)]
          & = \alpha-1 +\Phi\del[2]{\frac{\psi-r}{\tau}} +\frac{\psi-r}{\tau}\phi\del[2]{\frac{\psi-r}{\tau}}-\frac{\psi-r}{\tau}\phi\del[2]{\frac{\psi-r}{\tau}}\\ &=\alpha - \Phi\del[2]{\frac{r-\psi}{\tau}},
    \end{align*}
which completes the proof
\end{proof}

\subsubsection{Boundedness of Bayes-ACI iterates}

\begin{lemma}
\label[lemma]{lem:mean_bound}
Assume that $r_t \in [0, R]$ for all time $t$.  Then Bayes-ACI with any initialization $\psi_{1,k}\in[-\eta_k\tau_k^2-\tau_k\Phi^{-1}(\alpha),  R+\eta_k\tau_k^2-\tau_k\Phi^{-1}(\alpha)]$ satisfies
    \begin{align}
    \label{eq:m_t_bound}
         \psi_{t,k}\in[-\eta_k\tau_k^2-\tau_k\Phi^{-1}(\alpha),  R+\eta_k\tau_k^2-\tau_k\Phi^{-1}(\alpha)]
    \end{align}
for any $t$.
\end{lemma}

\begin{proof}
By \cref{lem:cp_svb_one_step}, 
    \begin{align}
    \label{eq:update_diff}
          \abs{ \psi_{t+1,k} - \psi_{t,k}}
    \le \eta_k \tau_k^2 \abs{\alpha - \Phi\del{\frac{r_t - \psi_{t,k}}{\tau_k}}}
    \le \eta_k\tau_k^2.
    \end{align}
We prove the lemma by contradiction. Let $t$ be the smallest time such that \eqref{eq:m_t_bound} is violated.  Suppose that $ \psi_{t,k}>R+\eta_k\tau_k^2-\tau_k\Phi^{-1}(\alpha)$, then by \eqref{eq:update_diff}, we must have  $\psi_{t-1,k}>R-\tau_k\Phi^{-1}(\alpha)$. Thus, by the assumption that $r_{t-1}<R$, we have 
    \begin{align*}
        \Phi\del[2]{\frac{r_{t-1} - \psi_{t-1,k}}{\tau_k}}
        < \Phi\del[2]{\frac{R -R+\tau_k\Phi^{-1}(\alpha)}{\tau_k}}\le \alpha
    \end{align*}
and thus,
        \begin{align*}
        \psi_{t,k}
        = \psi_{t-1,k}
    - \eta_k \tau_k^2 \sbr[3]{\alpha - \Phi\del[2]{\frac{r_{t-1} - \psi_{t-1,k}}{\tau_k}}}
    \le \psi_{t-1,k}
    \le R+\eta_k\tau_k^2-\tau_k\Phi^{-1}(\alpha),
    \end{align*}
where the last inequality follows from that we define $t$ as the first time that the bound does not hold. This contradicts with our assumption that $ \psi_{t,k}>R+\eta_k\tau_k^2-\tau_k\Phi^{-1}(\alpha)$. 

We then suppose that $ \psi_{t,k}<-\eta_k\tau_k^2-\tau_k\Phi^{-1}(\alpha)$. By \eqref{eq:update_diff}, we must have  $\psi_{t-1,k}<-\tau_k\Phi^{-1}(\alpha)$. Thus, by the assumption that $r_{t-1}\ge0$. we have
        \begin{align*}
        \Phi\del[2]{\frac{r_{t-1} - \psi_{t-1,k}}{\tau_k}}\
        >  \Phi\del[2]{\frac{\tau_k\Phi^{-1}(\alpha)}{\tau_k}}\ge \alpha
    \end{align*}
 and thus,
        \begin{align*}
        \psi_{t,k}
        = \psi_{t-1,k}
        - \eta_k \tau_k^2 \sbr[3]{\alpha - \Phi\del[2]{\frac{r_{t-1} - \psi_{t-1,k}}{\tau_k}}}
        \ge \psi_{t-1,k}
        \ge -\eta_k\tau_k^2-\tau_k\Phi^{-1}(\alpha)
    \end{align*}
where the last inequality follows from that we define $t$ as the first time that the bound does not hold. This contradicts with our assumption that $ \psi_{t,k}<-\eta_k\tau_k^2-\tau_k\Phi^{-1}(\alpha)$.
\end{proof}

\subsubsection{Smoothing error for pinball loss}

\begin{lemma}
\label[lemma]{lem:pinball_smooth_error}
Let $\psi\in\R$, $\tau>0$, $r \in \mathbb{R}$ and $\alpha\in(0,1)$. Then we have
    \begin{align}
    \label{eq:pinball_smooth_error}
         0\le\E_{\theta\sim \N(\psi, \tau^2)}[\ell_\alpha(\theta,r)]-\ell_\alpha(\psi,r)\le \frac{\tau}{\sqrt{2\pi}}.
    \end{align}   
\end{lemma}

\begin{proof}
Since the pinball loss is convex, the lower bound in \eqref{eq:pinball_smooth_error} follows from Jensen's inequality. We now derive the upper bound.  Since $\ell_\alpha(\psi,r)=(1-\alpha)(r-\psi)+(\psi-r)\ind(\psi>r)$, by \cref{lem:pinball-grad}, we have
    \begin{align}
    \label{eq:pinball_jensen}
         \E_{\theta\sim \N(\psi, \tau^2)}[\ell_\alpha(\theta,r)]-\ell_\alpha(\psi,r)
          = (\psi - r)\cbr{\Phi\del[2]{\frac{\psi - r}{\tau}}-\ind(\psi>r)}     + \tau\phi\del[2]{\frac{\psi - r}{\tau}}.
    \end{align}  
Define a function 
    \begin{align*}
        g(z)=\tau z\cbr{\Phi(z)-\ind(z>0)}     + \tau\phi(z)
    \end{align*}
so that the right-hand side of \cref{eq:pinball_jensen} is written as $g((\psi-r)/\tau)$.
Then for $z>0$,
    \begin{align*}
        g'(z) &=\tau(\Phi(z)-1)-\tau z\phi(z) -\tau\phi'(z)=\tau(\Phi(z)-1)<0
    \end{align*}
while for $z<0$, similarly,
    \begin{align*}
        g'(z) &=\tau\Phi(z)-\tau z\phi(z) -\tau\phi'(z)=\tau\Phi(z)>0.
    \end{align*}
These imply that $g$ is maximized at $z=0$. Hence, the right-hand side of \cref{eq:pinball_jensen} is bounded by $g(0)=\tau\phi(0)=\tau/\sqrt{2\pi}$.
\end{proof}

\subsection{Proof of \cref{lem:cp_svb_one_step}}

\begin{proof}
With $\tau_k^2$ fixed, the KL between two Gaussians of equal variance reduces to
$\mathrm{KL}(q_m, q_{m'}) = (m - m')^2 / (2 \tau_k^2)$.
The optimization problem becomes quadratic in $\psi$ with solution
    \begin{align*}
          \psi_{t+1,k} =  \psi_{t,k} - \eta_k \tau_k^2  \frac{\d}{\d\psi}\bar{L}_t( \psi_{t,k}),
    \end{align*}
and \cref{lem:pinball-grad} gives the stated form.
\end{proof}

\subsection{Proof of \cref{thm:long_run_single}}

\begin{proof}
Define the per-expert coverage deviation
    \begin{align}
    \label{eq:per-expert_dev}
          D_{t,k} :=   \mathbb{E}_{\theta \sim q_{t,k}}\bigl[\ind(r_t > \theta)\bigr]-\alpha
          =\Phi\del[2]{\frac{r_t -  \psi_{t,k}}{\tau_k}} - \alpha.
    \end{align}
Then by \cref{lem:cp_svb_one_step},
\[
  \psi_{t+1,k} -  \psi_{t,k} = \eta_k \tau_k^2 D_{t,k}.
\]
Summing over $t = 1, \ldots, T$, we have
\[
  \eta_k \tau_k^2 \sum_{t=1}^T D_{t,k}
  = \psi_{t+1,k} - \psi_{1,k}.
\]
\cref{lem:mean_bound} gives $|\psi_{t+1,k} - \psi_{1,k}| \leq R+2\eta_k\tau_k^2$, so
\begin{equation}
  \left| \sum_{t=1}^T D_{t,k} \right|
  \leq \frac{R}{\eta_k \tau_k^2}+2,
  \label{eq:per-expert-bound}
\end{equation}
which completes the proof.
\end{proof}

\subsection{Proof of \cref{thm:long_run_agg}}

\begin{proof}
The proof follows the strategy of Theorem~3.2 of \cite{gibbs2024conformal} with modifications need to handle our setting. Let the aggregated coverage deviation be denoted by
    \begin{align*}
         D_t := \mathbb{E}_{\theta \sim q_t}[\ind(r_t > \theta)] - \alpha
  = \sum_{k=1}^K \omega_{t,k} D_{t,k}.
    \end{align*}
where $D_{t,k}$ is defined in \eqref{eq:per-expert_dev}. Define
    \begin{align*}
        \tilde \psi_t := \sum_{k=1}^K \frac{\omega_{t,k} (\psi_{t,k} + \tau_k \Phi^{-1}(\alpha))}{A_k}.
    \end{align*}
Observe that
\begin{align*}
\tilde \psi_{t} 
    &= \sum_{k=1}^K  \frac{\omega_{t,k} ( (\psi_{t+1,k} -A_k D_{t,k}) + \tau_k \Phi^{-1}(\alpha) )}{A_k} \\
    &= \sum_{k=1}^K \frac{\omega_{t,k} (\psi_{t+1,k} + \tau_k\Phi^{-1}(\alpha) )}{A_k} -  \sum_{k=1}^K \omega_{t,k}D_{t,k} \\
    &= \tilde{\psi}_{t+1} + \sum_{k=1}^K  \frac{(\omega_{t,k}- \omega_{t+1,k}) (\psi_{t+1,k} + \tau_k \Phi^{-1}(\alpha))}{A_k} - \sum_{k=1}^K \omega_{t,k} D_{t,k}
\end{align*}
Thus,
\begin{align}
\label{eq:err_agg}
    D_t = \sum_{k=1}^K \omega_{t,k} D_{t,k}
    =  \tilde{\psi}_{t+1} - \tilde{\psi}_{t} + \sum_{k=1}^K \frac{(\omega_{t,k}- \omega_{t+1,k})(\psi_{t+1,k} + \tau_k \Phi^{-1}(\alpha))}{A_k} 
\end{align}
Recall that by definition,
\begin{align*}
    \omega_{t+1,k} = (1-\sigma_t) \bar{\omega}_{t+1,k} + \sigma_t /K
\end{align*}
By a direct computation,
\begin{align*}
\bar\omega_{t,k} - \omega_{t,k}
&= 
\frac{\omega_{t,k}\exp(-\gamma_t  \bar{L}_{t,k}(\psi_{t,k}))}
{\sum_{h=1}^K \omega_{t,h} \exp(-\gamma_t\bar{L}_{t,h}(\psi_{t,h}))} - \omega_{t,k} \\
&= \omega_{t,k}
\frac{\exp(-\gamma_t  \bar{L}_{t,k}(\psi_{t,k}))-\sum_{h=1}^K \omega_{t,h} \exp(-\gamma_t\bar{L}_{t,h}(\psi_{t,h}))}
{\sum_{h=1}^K \omega_{t,h} \exp(-\gamma_t\bar{L}_{t,h}(\psi_{t,h}))} \\
&= \omega_{t,k}
\frac{\sum_{h=1}^K \omega_{t,h}(\exp(-\gamma_t  \bar{L}_{t,k}(\psi_{t,k}))- \exp(-\gamma_t\bar{L}_{t,h}(\psi_{t,h})))}
{\sum_{h=1}^K \omega_{t,h} \exp(-\gamma_t\bar{L}_{t,h}(\psi_{t,h}))} \\
&= \omega_{t,k} \frac{\sum_{h=1}^K \omega_{t,h} \exp(-\gamma_t\bar{L}_{t,h}(\psi_{t,h}))\{\exp(\gamma_t\bar{L}_{t,h}(\psi_{t,h})-\gamma_t  \bar{L}_{t,k}(\psi_{t,k})) - 1\}}{\sum_{h=1}^K \omega_{t,h} \exp(-\gamma_t\bar{L}_{t,h}(\psi_{t,h}))} \\
&=\omega_{t,k} \sum_{h=1}^K \bar{\omega}_{t+1,h}(\exp(\gamma_t\bar{L}_{t,h}(\psi_{t,h})-\gamma_t  \bar{L}_{t,k}(\psi_{t,k})) - 1)
\end{align*}
By the triangular inequality,
    \begin{align*}
        &|\bar{L}_{t,h}(\psi_{t,h}) -  \bar{L}_{t,k}(\psi_{t,k})|\\
        &\le |\bar{L}_{t,h}(\psi_{t,h})-\ell_\alpha(\psi_{t,h},r_t)|
        +|\bar{L}_{t,k}(\psi_{t,k})-\ell_\alpha(\psi_{t,k},r_t)|
        +|\ell_\alpha(\psi_{t,h},r_t)-\ell_\alpha(\psi_{t,k},r_t)|.
    \end{align*}
The first two terms bounded above by $\tau_{\max}/\sqrt{2\pi}\le \tau_{\max}/2$ by \cref{lem:pinball_smooth_error}. For the last term, since we know that $\psi_{t,k} \in [-\eta_k \tau_k^2 - \tau_k\Phi^{-1}(\alpha), R +\eta_k\tau_k^2 - \tau_k \Phi^{-1}(\alpha)]$ by \cref{lem:mean_bound}, we have
    \begin{align*}
        |\ell_\alpha(\psi_{t,h},r_t)-\ell_\alpha(\psi_{t,k},r_t)|
       & \le \max \{ \alpha , 1-\alpha \}|\psi_{t,h} - \psi_{t,k}|\\
        &\le R + 2 A_{\max} + (\tau_{\max} - \tau_{\min}) |\Phi^{-1}(\alpha)|.
    \end{align*}
Thus,
$$
|\bar L_{t,h}(\psi_{t,h}) - \bar L_{t,k}(\psi_{t,k})| \le  \tau_{\max} + R + 2 A_{\max} + (\tau_{\max} - \tau_{\min}) |\Phi^{-1}(\alpha)| = C_{\alpha}.
$$
Hence, by the mean value theorem,
$$
|\exp(\gamma_t\bar{L}_{t,h}(\psi_{t,h})-\gamma_t  \bar{L}_{t,k}(\psi_{t,k})) - 1| \le \gamma_t C_{\alpha} \exp( \gamma_t C_{\alpha})
$$
and thus also,
$$
|\bar\omega_{t,k}- \omega_{t,k}| \le \omega_{t,k}\gamma_t C_{\alpha} \exp( \gamma_t C_{\alpha}).
$$
Applying \cref{lem:mean_bound} again with the triangular inequality, we conclude that
\begin{align*}
    &\left|\sum_{k=1}^K \frac{(\omega_{t+1,k}- \omega_{t,k}) (\psi_{t+1,k} + \tau_k \Phi^{-1}(\alpha))}{A_k}\right| \\
    &\le (1- \sigma_t)\sum_{k=1}^K \left| \frac{(\bar\omega_{t,k} - \omega_{t,k})(\psi_{t+1,k} + \tau_k\Phi^{-1}(\alpha))}{A_k} \right| + \sigma_t\sum_{k=1}^K \left| \frac{(1/K - \omega_{t,k})(\psi_{t+1,k} + \tau_k \Phi^{-1}(\alpha))}{A_k} \right| \\
    &\le \frac{R + A_{\max}}{A_{\min}}\gamma_t C_{\alpha}\exp( \gamma_t C_{\alpha}) + 2 \sigma_t \frac{R+ A_{\max}}{A_{\min}}
\end{align*}
where we used the fact that 
$$
\left| \frac{\psi_{t,k}+ \tau_k \Phi^{-1}(\alpha)}{A_k} \right| \le \frac{R+ A_{\max}}{A_{\min}}.
$$
Taking a sum over $t$ in \cref{eq:err_agg} gives us,
\begin{align*}
&\left|
\frac{1}{T} \sum_{t=1}^{T}
\mathbb{E}_{\theta \sim q_t}\bigl[\ind(r_t > \theta)\bigr] - \alpha
\right| \\
&\le 
\frac{|\tilde \psi_{1} - \tilde \psi_{T+1}|}{T} + \frac{(R + A_{\max})C_{\alpha}}{A_{\min}}\frac{1}{T} \sum_{t=1}^T \gamma_t \exp({\gamma_t C_{\alpha}}) + 2\frac{R+ A_{\max}}{A_{\min}} \frac{1}{T} \sum_{t=1}^T \sigma_t 
\end{align*}
Again, by \cref{lem:mean_bound}, we can observe that $\tilde \psi_t \in [-1,  1 + R/A_{\min}]$ and thus $|\tilde \psi_1 - \tilde\psi_{T+1}| \le 2 + R / A_{\min}$. Plugging this into the previous expression yields the desired upper bound. 
\end{proof}

\section{Proofs for \cref{sec:regression}}

\subsection{Proof of \cref{lem:log-loss-oracle}}

\begin{proof}
Applying \cref{thm:regret_noshare} with $\ell_t(f) = -\log \phi_\varsigma(Y_t-f(X_t))$, we have, for every $k\in[K]$,
\begin{align*}
	-\sum_{t=1}^T
	\log \p_t(Y_t|X_t,\cD_{t-1})
  	&\le	-\sum_{t=1}^T\log \del[1]{ \p_{t,k}(Y_t|X_t,\cD_{t-1})}
	+	\log(1/\omega_{1,k}).
\end{align*}
Adding $\sum_{t=1}^T\log \p_\star(Y_t|X_t)$ to both sides and taking
expectation gives
\begin{align*}
	&\sum_{t=1}^T	\E_{\cD_t}\sbr{	\log\frac{	\p_\star(Y_t|X_t)	}{	\p_t(Y_t|X_t,\cD_{t-1})	}	} \\
	&\qquad\le
	\sum_{t=1}^T
	\E_{\cD_t}\sbr{	\log\frac{\p_\star(Y_t|X_t)}{	\p_{t,k}(Y_t|X_t,\cD_{t-1})	}	}
	+		\log(1/\omega_{1,k}).
\end{align*}
Since $\p_t$ and $\p_{t,k}$ are $\cD_{t-1}$-measurable conditional densities and $(X_t,Y_t)$ is independent of $\cD_{t-1}$, the left- and right-hand sides are the corresponding expected conditional KL divergences. Taking the minimum over $k$ completes the proof.
\end{proof}

\subsection{Proof of \cref{thm:gp_single}}

\begin{proof}
\noindent
\textbf{Step 1: Regret bound} The marginal likelihood under the prior $\pi_{a_k}$ is
    \begin{align*}
        m_k(\cD_T)  =  \int \prod_{t=1}^T \phi_\varsigma   \bigl(Y_t - f(X_t)\bigr)  \d\pi_{a_k}(f).
    \end{align*}
By the definition of conditional density and Bayes' rule applied recursively,
    \begin{equation}\label{eq:chain}
    \sum_{t=1}^T \log \p_{t,k}(Y_t | X_t, \cD_{t-1})  =  \log m_k(\cD_T).
    \end{equation}
Since $\p_{t,k}(\cdot|X_t,\cD_{t-1})$ is $(X_t,\cD_{t-1})$-measurable, by taking expectations and conditioning on $(X_t,\cD_{t-1})$ in each summand, we have
\begin{equation}\label{eq:kl_step}
    \E_{\cD_{t}}\sbr{\log\frac{\p_\star(Y_t| X_t)}{\p_{t,k}(Y_t| X_t, \cD_{t-1})}}
    =\E_{\cD_{t}}\sbr{\kl(\p_\star(\cdot| X_t) , \p_{t,k}(\cdot| X_t, \cD_{t-1}))}.
\end{equation}
By \eqref{eq:chain} and \eqref{eq:kl_step},
    \begin{align}
     \sum_{t=1}^T &\E_{\cD_{t}}\sbr{\kl(\p_\star(\cdot| X_t), \p_{t,k}(\cdot| X_t, \cD_{t-1}))}\notag\\
   & = \E_{\cD_{t}}\left[-\log m_k(\cD_T)\right] 
    + \sum_{t=1}^T\E_{\cD_t}\left[\log \p_\star(Y_t| X_t)\right].\label{eq:kl-sum}
    \end{align}
Since $Y_t | X_t \sim \N(f_\star(X_t), \varsigma^2)$, the second term in \eqref{eq:kl-sum} has a closed form given by
    \begin{align*}
             \E_{\cD_t}\left[\log p^\star(Y_t | X_t)\right] = -\frac{1}{2}\log(2\pi\varsigma^2) - \frac{1}{2}.
    \end{align*}
For the first term, the Donsker-Varadhan variational formula gives, for any probability measure $q$ absolutely continuous with respect to $\pi_{a_k}$,
    \begin{align*}
        \log m_k(\cD_T) &= \log \int\exp\del[2]{\sum_{t=1}^T \log\phi_\varsigma (Y_t - f(X_t))} \d\pi_{a_k}(f) \\
        &\ge \E_{f\sim q}\left[\sum_{t=1}^T \log\phi_\varsigma \bigl(Y_t - f(X_t)\bigr)\right] 
        - \kl(q , \pi_{a_k}).
    \end{align*}
Since $\E_{Y,X}[\log\phi_\varsigma (Y - f(X))] = -\frac{1}{2}\log(2\pi\varsigma^2) - \frac{1}{2} - \frac{1}{2\varsigma^2}\norm{f - f_\star}_{\cL_2(P_X)}^2$, we have
\begin{equation}\label{eq:dv-applied}
     \E_{\cD_t}[-\log m_k(\cD_T)] 
    \le T\left[\frac{1}{2}\log(2\pi\varsigma^2) + \frac{1}{2}\right] 
    + \frac{T}{2\varsigma^2}\E_q\sbr{\norm{f - f_\star}_{\cL_2(P_X)}^2} + \kl(q ,\pi_{a_k})
\end{equation}
Substituting \eqref{eq:dv-applied} into \eqref{eq:kl-sum}, the constants $T[\frac{1}{2}\log(2\pi\varsigma^2) + \frac{1}{2}]$ are canceled out, leaving
    \begin{align}
      \sum_{t=1}^T \E_{\cD_{t}}\sbr{\kl(\p_\star(\cdot| X_t), p_{t,k}(\cdot| X_t, \cD_{t-1})}
      \le \frac{T}{2\varsigma^2}  \E_{f\sim q}\sbr{\norm{f - f_\star}_{\cL_2(P_X)}^2} + \kl(q , \pi_{a_k}).  \label{eq:after-cancel}
    \end{align}
    
\noindent
\textbf{Step 2: GP prior concentration analysis}
It remains to choose $q$ so that both terms on the right-hand side of \eqref{eq:after-cancel} are  bounded suitably.

\begin{lemma}[Prior concentration]\label{lem:vdv}
Assume the true regression function satisfies $f_\star \in \cC^\beta( [0,1]^d)$ with $\norm{f_\star}_\infty \le B$. Let $\pi_a$ denote the law of the rescaled Gaussian process $W^a$ with squared-exponential kernel of inverse bandwidth $a$.  Then there exists a constant $C>0$ such that,  for any $\varepsilon > Ca^{-\beta}$, the prior mass of the $\cL_2(P_X)$-ball of radius $\varepsilon$ around $f_\star$ satisfies
\begin{equation}
    -\log \pi_a\del[1]{ \norm{f - f_\star}_{\cL_2(P_X)} \le \varepsilon}
    \lesssim a^d \log^{1+d}(a/\varepsilon).
\end{equation}
In particular, setting
\begin{equation}
\varepsilon_T(a): =C\cbr[1]{ T^{-1} a^d +   a^{-2\beta}}^{1/2}
\end{equation}
yields
\begin{equation}\label{eq:prior-mass}
     -\log \pi_a\del[1]{ \norm{f - f_\star}_{\cL_2(P_X)} \le \varepsilon}
    \le \varepsilon_T(a)\}\bigr) \lesssim T \varepsilon_T(a)^2 \log^{1+d}(T).
\end{equation}

\end{lemma}

For the proof of the above lemma, see  Section 5.1 of \cite{van2009adaptive}.

Now take $q^*$ to be the prior $\pi_{a_k}$ restricted (and renormalized) to the ball $B(f_\star, \varepsilon_T(a)) = \{f : \norm{f - f_\star}_{\cL_2(P_X)} \le \varepsilon_T(a)\}$:
    \begin{align*}
          q^*(\d f) = \frac{\ind\{f \in \cB_k\}}{\pi_{a_k}(\cB_k)}\pi_{a_k}(\d f)
          \text{ with }\cB_k:=\cbr{f: \norm{f - f_\star}_{\cL_2(P_X)} \le \varepsilon_T(a_k)}
    \end{align*}
Then by construction,
    \begin{align*}
        \E_{f\sim q^*}[\norm{f - f_\star}_{\cL_2(P_X)}^2] &\le\varepsilon_T(a_k)^2
    \end{align*}
and
    \begin{align*}
    \kl(q^* , \pi_{a_k}) &= -\log\pi_{a_k}(\cB_k) \lesssim T \varepsilon_T(a_k)^2.
    \end{align*}
These two observations complete the proof.
\end{proof}

\bibliography{_references}

\end{document}